\documentclass[11pt]{article}
	
	\newcommand{\blind}{0}
	
    \makeatletter
    \renewcommand\section{\@startsection {section}{1}{\z@}%
                                       {-3.5ex \@plus -1ex \@minus -.2ex}%
                                       {2.3ex \@plus.2ex}%
                                       {\normalfont\fontfamily{phv}\fontsize{16}{19}\bfseries}}
    \renewcommand\subsection{\@startsection{subsection}{2}{\z@}%
                                         {-3.25ex\@plus -1ex \@minus -.2ex}%
                                         {1.5ex \@plus .2ex}%
                                         {\normalfont\fontfamily{phv}\fontsize{14}{17}\bfseries}}
    \renewcommand\subsubsection{\@startsection{subsubsection}{3}{\z@}%
                                        {-3.25ex\@plus -1ex \@minus -.2ex}%
                                         {1.5ex \@plus .2ex}%
                                         {\normalfont\normalsize\fontfamily{phv}\fontsize{14}{17}\selectfont}}
    \makeatother
	
	\usepackage{amsmath}
	\usepackage{graphicx}
	\usepackage{enumerate}
	\usepackage{natbib} 
	\usepackage{url} 
		\usepackage{epstopdf}
\usepackage[caption=false]{subfig}
\usepackage[compact]{titlesec}
\usepackage{natbib}
\usepackage{siunitx}
\usepackage{tabularx} 
\usepackage{amsmath}
\usepackage{bm}
\usepackage{amsthm}
\newtheorem{theorem}{Theorem}

\newtheorem{proposition}[theorem]{Proposition}

\newtheorem{definition}[theorem]{Definition}

\usepackage{amsfonts}
\usepackage{amssymb}
\usepackage{mathtools}
\usepackage{bbm}
\usepackage{xcolor,soul,framed} 
\soulregister\cite7
\soulregister\citep7 
\soulregister\citet7
\soulregister\ref7 
\soulregister\eqref7 
\soulregister\pageref7 

\colorlet{shadecolor}{yellow}
\usepackage{graphicx} 
\usepackage{booktabs}

\usepackage{hyperref} 
\hypersetup{
	colorlinks=true,       
	linkcolor=red,        
	citecolor=blue,        
	filecolor=magenta,     
	urlcolor=black
}
\usepackage{algorithm,algpseudocode}
\usepackage{eqparbox}
\usepackage{xcolor}
\usepackage{enumitem}   
\usepackage{multirow}
\usepackage{multicol}
	
\newcommand{\bmTheta}{\bm{\Theta}}
	
\begin{document}
		
		\def\spacingset#1{\renewcommand{\baselinestretch}%
			{#1}\small\normalsize} \spacingset{1}
		
		\if0\blind
		{
			\title{\bf Self-scalable Tanh (Stan): Faster Convergence and Better Generalization in Physics-informed\\ Neural Networks}
			\author{Raghav~Gnanasambandam\textsuperscript{a}, Bo~Shen\textsuperscript{a}, Jihoon~Chung\textsuperscript{a}, Xubo~Yue\textsuperscript{b},  \\
			   and Zhenyu (James) Kong\textsuperscript{a*} 	\\
			\textsuperscript{a}Grado Department of Industrial and Systems Engineering, Virginia Tech, Blacksburg, US \\
           	\textsuperscript{b}Industrial and Operations Engineering, University of Michigan, Ann Arbor, US \hfill \hfill\\
           		\textsuperscript{*}Corresponding Author \hfill \hfill} 
           	
			\date{}
			\maketitle
				\vspace{-0.5cm}
		} \fi
		
		\if1\blind
		{

            \title{\bf \emph{IISE Transactions} \LaTeX \ Template}
			\author{Author information is purposely removed for double-blind review}
			
\bigskip
			\bigskip
			\bigskip
			\begin{center}
				{\LARGE\bf \emph{IISE Transactions} \LaTeX \ Template}
			\end{center}
			\medskip
		} \fi
		\bigskip
		
	\begin{abstract}
Physics-informed Neural Networks (PINNs) are gaining attention in the engineering and scientific literature for solving a range of differential equations with applications in weather modeling, healthcare, manufacturing, etc. Poor scalability is one of the barriers to utilizing PINNs for many real-world problems. To address this, a Self-scalable tanh (Stan) activation function is proposed for the PINNs. The proposed Stan function is smooth, non-saturating, and has a trainable parameter. During training, it can allow easy flow of gradients to compute the required derivatives and also enable systematic scaling of the input-output mapping. It is shown theoretically that the PINNs with the proposed Stan function have no spurious stationary points when using gradient descent algorithms. The proposed Stan is tested on a number of numerical studies involving general regression problems. It is subsequently used for solving multiple forward problems, which involve second-order derivatives and multiple dimensions, and an inverse problem where the thermal diffusivity of a rod is predicted with heat conduction data.  These case studies establish empirically that the Stan activation function can achieve better training and more accurate predictions than the existing activation functions in the literature.            
	\end{abstract}
			
	\noindent%

	\spacingset{1.5} 

\section{Introduction} \label{sec: intro}
Deep learning has revolutionized many facets of our modern society such as self-driving vehicles, recommendation systems, and web search tools. The fundamental ingredients of deep learning were first introduced and developed by \cite{mcculloch1943logical} and \cite{rosenblatt1958perceptron} half a century ago. Since then, this work has spurred continued development and exploration of the neural network (NN). The triumph of the NN can be explained by the seminal universal approximation theorem \citep{hornik1989multilayer}: a properly constructed NN can act as a universal approximator for any continuous and bounded functions. As a result, the NN can be used for a wide variety of function approximations, including but not limited to image segmentation, machine translation, speech recognition, weather prediction, etc. 

Despite the successes, when analyzing complex engineering systems, NNs are almost purely data-driven. They inevitably ignore a vast amount of prior knowledge, such as the principled physical laws (typically in the form of differential equations). Such prior knowledge is useful in two aspects: (1) in the \textit{small data} regime, it encodes valuable information that amplifies the information content of the data~\citep{karniadakis2021physics} and (2) it acts as a regularization agent that eliminates inadmissible solutions of complex physical processes~\citep{nabian2020physics}.  Physics-informed Neural Network (PINN)~\citep{raissi2019physics} is proposed to encode any given laws of physics described by general  differential equations into the neural network.  Accordingly, PINN is an excellent candidate for solving numerical differential equations, integrations, and dynamical systems  in many engineering applications~\citep{lagaris1998artificial, rudd2013solving, rudd2015constrained, raissi2018hidden,  blechschmidt2021three}, which has the advantages of being mesh-free and computationally efficient compared to numerical methods like Finite Element Methods~\citep{donea2003finite}. Indicatively, PINN has found applications in many real-world problems \citep{mao2020physics, sahli2020physics, cai2021physics}. For instance, \cite{cai2021physics} employ PINNs to solve general heat transfer problems and power electronics problems of industrial complexity. \cite{sahli2020physics} proposed a PINN that incorporates wave propagation dynamics and successfully applied it to Cardiac Activation Mapping problems.


PINNs are structurally and functionally similar to  multi-layered fully-connected neural networks~\citep{karniadakis2021physics}. Like NNs, they are designed for the mapping between the input (which are usually the spatial and/or temporal location for PINNs) and the output (which is a function of the input) to learn the data-driven solutions of differential equations. The key difference is that the physical laws (i.e., differential equations) are embedded into the loss function together with the fitting loss of available data (for example, initial/boundary conditions or experimental data)~\citep{raissi2019physics, karniadakis2021physics}. 

However, the training of the PINN is often difficult and not understood well enough~\citep{wang2021understanding,wang2022and}. There are two main issues that create the difficulty: (1) the analysis by~\cite{wang2021understanding} shows that a primary failure mode for PINN is the vanishing gradients when using the gradient descent algorithms. The vanishing gradients are caused by the saturation, which is introduced in Section~\ref{subsec: Proposed Activation function}. (2) the unknown magnitude of the outputs makes the training unstable. Particularly, higher orders of magnitude of outputs arise in many engineering problems. For example, a rod of unit length can take temperature values from 0 to 2000 degrees Celsius in a heat transfer process. In a traditional supervised learning task with NN, the data can always be scaled (i.e., normalization) before training~\citep{Goodfellow-et-al-2016}. However, it is impossible for PINN to know the magnitude of outputs in advance with the given domain, differential equations, and initial/boundary conditions. The idea of normalizing the output is absurd for the PINN case and thus, PINN~\citep{raissi2019physics} more often fails to solve the problems where the outputs are of a higher order of magnitude than the inputs. 


To overcome these two issues, we look into the activation function as a  solution, which is the key component to introducing the non-linearity into NNs~\citep{haykin1994neural} and is essential to the training of NNs~\citep{hayou2019impact}. When embedding physical laws into loss function, finding the derivatives of the output with respect to input is a crucial step in finding the loss function value itself~\citep{raissi2019physics}. This makes the activation function play a much more significant role in gradient flow in PINN  than the usual NNs~\citep{wang2021understanding}. There are several popular activation functions, such as tanh, ReLU, sigmoid, etc~\citep{ranjan2020understanding}. In PINN, it is preferred to use smooth functions~\citep{cai2021physics,zhu2021machine,zobeiry2021physics} over the widely used ReLU-like non-smooth activation functions~\citep{mishra2020estimates}, where tanh has been shown markedly successful~\citep{jagtap2020conservative, jagtap2020extended, singh2019pi, zhang2021physics}. However, tanh based activation functions still suffer from  the above-mentioned issues. To improve from tanh,  a new activation function, namely, Self-scalable tanh (Stan), is proposed for PINN in this work. The Stan function has a tanh term and an additional self-scaling term similar to Swish function~\citep{ramachandran2017searching}, along with a trainable parameter. To summarize, the contributions of this paper are as follows: 
\begin{itemize}
\item Propose the Self-scalable tanh (Stan) activation function, which enables learning outputs (solution) with values of higher order magnitude.
\item Analyze the theoretical convergence of PINN with Stan (i.e., proposed method) using gradient descent algorithms, where our proposed method is not attracted to the spurious stationary points\footnote{A stationary point is spurious if it is not a global minimum~\citep{no2021wgan}.}.
\item In both forward problems (the approximate solutions are obtained) and inverse problems (coefficients involved in the governing equation are identified), our proposed method shows faster training convergence and better test generalization than state-of-the-art activation functions.
\end{itemize}

\subsection{Related Work in Activation Functions}
The activation function and its gradient play an important role in the training process influencing the gradients of the loss function for optimization of parameters~\citep{hayou2019impact}. A badly designed activation function causes the loss of information of the input during forward propagation and the exponential vanishing/exploding of gradients during back-propagation, leading to poor convergence. Designing  activation functions for NNs has been an active research area~\citep{ranjan2020understanding}. Since no single activation function works well in all situations, it is desirable to design a data-driven or adaptive activation function that learns the best possible activation function given a specific dataset \citep{jagtap2020adaptive}. For instance, \cite{ramachandran2017searching} use a reinforcement learning-based approach to search for the best combination of activation functions. Based on their experimental results, they propose an activation function Swish that achieves state-of-the-art performance across numerous datasets. \cite{misra2019mish} then propose Mish: a self-gating version of Swish and further show that Mish yields promising results in many computer vision tasks. In the literature of PINN, \cite{jagtap2020adaptive, jagtap2020locally} developed an adaptive activation function framework that introduces multiple parameters to dynamically change the topology of the loss function during the optimization process. They have shown that this approach can significantly accelerate the training process, improve convergence rate, and yield more accurate solutions than the traditional activation functions. 

\textbf{Paper Organization:} The remainder of this paper is organized as follows. The proposed methodology with a Self-scalable tanh activation function is introduced in Section~\ref{sec: proposed activation}, where  the motivation and properties of the proposed activation function are discussed. Section~\ref{sec: numerical study} provides the results on neural network approximations of smooth and discontinuous functions on a typically high magnitude of output, followed by PINN-based forward and inverse problems in Section~\ref{sec: PINN study} for testing and validation of the proposed activation function.  Finally, the conclusions and future work are discussed in Section~\ref{sec: conclusion}.

\section{Proposed Method} \label{sec: proposed activation}
Consider the general form of partial differential equations (PDE) written as
\begin{align*} \label{PDE general form}
\begin{split}
    L[u(\bm{x})] &= f(\bm{x}), \,\, \bm{x} \in \Omega, \\
    u(\bm{x}) &= g(\bm{x}), \,\, \bm{x} \in \partial \Omega,     
\end{split}    
\end{align*}
where $L$ represent the linear/nonlinear differential operator, $f$ and $g$ are known functions, and $u$ is the unknown solution to be learned with initial/boundary condition $u(\bm{x}) = g(\bm{x}), \ \bm{x} \in\partial \Omega$. To find the solution $u$, the framework of the Physics-informed Neural Network (PINN) will be first introduced in Section~\ref{subsec: PINN}. The proposed activation function will be subsequently explained in Section~\ref{subsec: Proposed Activation function}.

\subsection{Physics-informed Neural Networks} \label{subsec: PINN}
We consider a NN of depth $D$ corresponding to a network with an input layer, $D-1$ hidden layers, and an output layer. In the $k$th hidden layer, $N_k$  neurons are present. Each hidden layer of the network receives an output $\bm{z}_{k-1} \in \mathbb{R}^{N_{k-1}}$ from the previous layer where an affine transformation of the form
\begin{equation} \label{eq: one layer}
    \mathcal{L}_k(\bm{z}_{k-1})\coloneqq \bm{W}^k\bm{z}_{k-1} +\bm{b}^k
\end{equation}
is performed. The network weights $\bm{W}^k \in \mathbb{R}^{N_k \times N_{k-1}}$ and bias term $\bm{b}^k \in \mathbb{R}^{N_k}$ associated with the $k$th layer are initialized  with independent and identically distributed (iid) samples. The nonlinear activation function $\sigma(\cdot)$ is applied to each component of the transformed vector before sending it as an input to the next layer. Based on the definition in Eq.~\eqref{eq: one layer}, the final neural network representation is given by the composition
\begin{equation} \label{eq: nn structure}
    u_{\bmTheta}(\bm{z}) = (\mathcal{L}_D\circ \sigma \circ \mathcal{L}_{D-1}\circ \dots \sigma \circ \mathcal{L}_1)(\bm{z}),
\end{equation}
where the operator ``$\circ$'' is the composition operator,  $\bmTheta=\{\bm{W}^k,\bm{b}^k\}_{k=1}^D$ represents the trainable parameters in the network, $u$ is the output and $\bm{z}^0=\bm{z}$ is the input.

\begin{figure}[!ht]\vspace{-0.0cm}
	\centering
	\includegraphics[width=\textwidth]{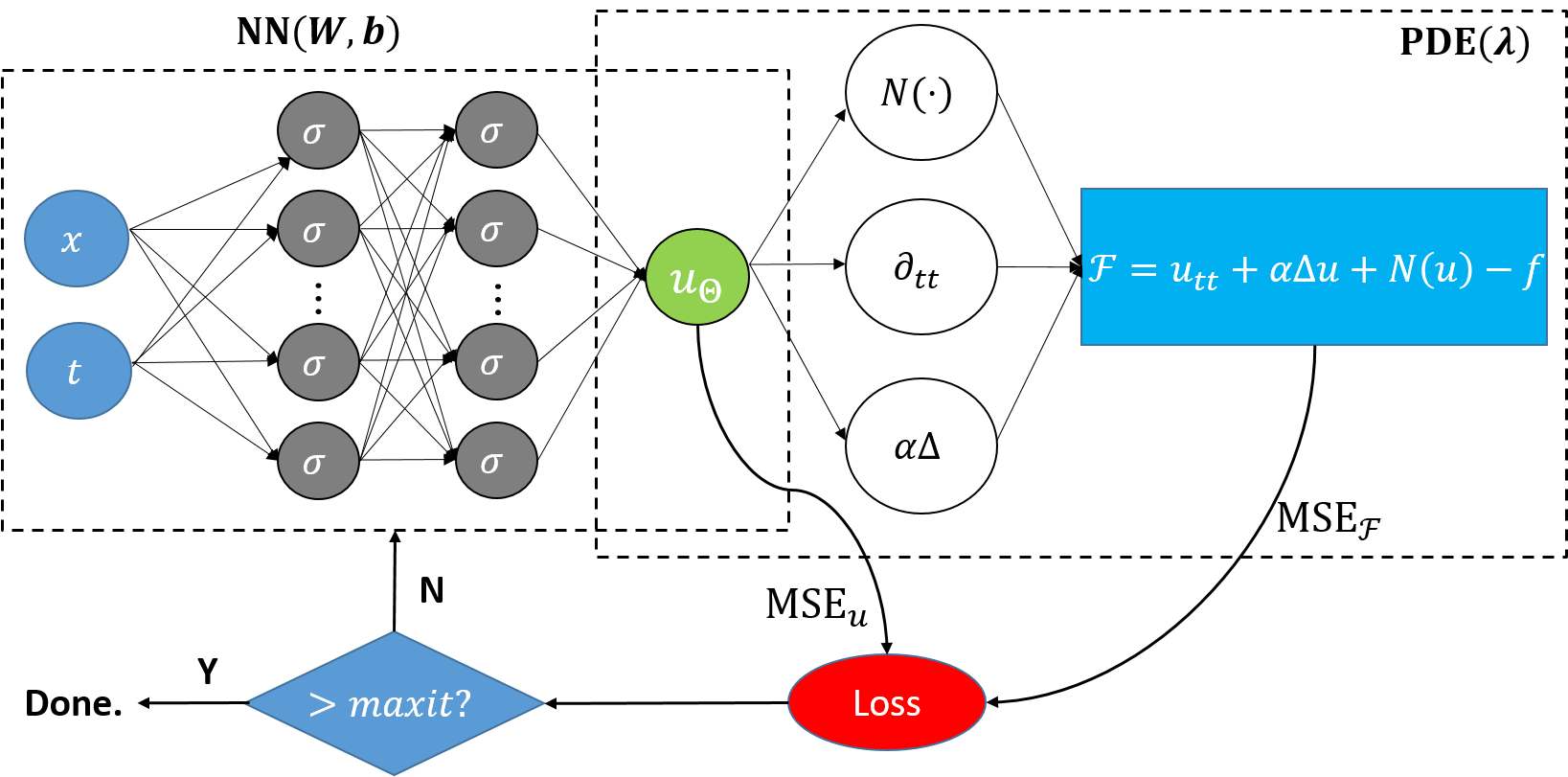}\vspace{-0.9cm}
	\caption{Schematic of PINN for  the  Klein-Gordon equation. The left
(physics-uninformed) network represents the surrogate of the PDE solution $u_{\bmTheta}(x,t)$ while the right (physics-informed) segment describes the PDE residual $\mathcal{F}=\frac{\partial^2 u_{\bmTheta}}{\partial t^2}+\alpha \Delta u_{\bmTheta}+N(u_{\bmTheta})-f$ calculation. The two parts share hyper-parameters, and they both contribute to the loss function. The loss
function includes a supervised loss of data measurements of $u(x,t)$ from the initial and
boundary conditions, namely, $\textnormal{MSE}_u$, and an unsupervised loss of PDE, namely, $\textnormal{MSE}_{\mathcal{F}}$.}
\label{fig: PINNschematic}
\end{figure}

PINN~\citep{raissi2019physics} uses the NN structure in Eq.~\eqref{eq: nn structure} to find the solution for the differential equation provided with the initial and boundary conditions. 
To demonstrate it, an example is presented in Figure~\ref{fig: PINNschematic} to solve the nonlinear Klein-Gordon equation using PINN.  Klein-Gordon equation is a second-order hyperbolic partial differential equation arising in many scientific fields like soliton dynamics and condensed matter physics~\citep{caudrey1975sine}, solid-state physics, and nonlinear wave equations~\citep{dodd1982solitons}.  The in-homogeneous Klein-Gordon equation is given by
\begin{equation}\label{eq: Klein-Gordon}
   u_{tt}+\alpha \Delta u+N(u) = f(x,t), \ x\in[-1,1], \ t>0,
\end{equation}
where $\Delta$ is a Laplacian operator and $N(u) = \beta u + \delta u^k$ is the nonlinear term with quadratic nonlinearity ($k = 2$) and cubic non-linearity ($k = 3$); $\lambda =\{\alpha, \beta, \delta\}$ are constants for the PDE. The initial and the boundary conditions are given by $u(x,0) = x$, $u(-1,t) = -\cos(t)$, and $u(1,t) = \cos(t)$. Specifically, the PINN algorithm aims to learn a NN surrogate $u_{\bmTheta}$ for predicting the solution $u$ of the governing PDE as shown in Figure~\ref{fig: PINNschematic}.

The NN solution must satisfy the governing equation given by the residual $\mathcal{F}(u_{\bmTheta})=u_{{\bmTheta}_{tt}}+\alpha \Delta u_{\bmTheta}+N(u_{\bmTheta}) - f(x,t)$ evaluated at randomly chosen $N_f$ residual points in the domain which is used to calculate $\textnormal{MSE}_{\mathcal{F}}$ in Eq.~\eqref{eq: definition MSE_f}. To construct the residuals in the loss function, derivatives of the solution $u_{\bmTheta}$ with respect to the independent variables $x$ and $t$ are required, which can be computed using the \textit{Automatic Differentiation} (AD)~\citep{baydin2018automatic}. AD is an accurate way to calculate derivatives in a computational graph compared to numerical differentiation since they do not suffer from the errors such as truncation and round-off errors~\citep{baydin2018automatic}. Thus, the PINN method is a grid-free method, which does not require mesh for solving equations, which is a popular approach used by Finite Element Method~\citep{donea2003finite}. 

The main feature of the PINN is that it can easily incorporate all the given information (for example, governing equation, experimental data and initial/boundary conditions) into the loss function, thereby recast the original problem into an optimization problem.  In the PINN algorithm, the loss function is defined as
\begin{equation} \label{eq: PINN loss function}
J(\bmTheta) = w_{\mathcal{F}}\textnormal{MSE}_{\mathcal{F}}+w_u\textnormal{MSE}_u,
\end{equation}
where the mean squared error (MSE) is given as
\begin{align}
    \texttt{PDE residual loss: }  \textnormal{MSE}_{\mathcal{F}} & =\frac{1}{N_f}\sum_{i=1}^{N_f}\left| \mathcal{F}_{\bmTheta}(\bm{x}_f^i) \right|^2, \label{eq: definition MSE_f} \\ 
   \texttt{Data loss: }  \textnormal{MSE}_u & =\frac{1}{N_u}\sum_{i=1}^{N_u}\left| u(\bm{x}_u^i) - u_{\bmTheta}(\bm{x}_u^i)\right|^2, \label{eq: definition MSE_u}
\end{align}
where  $\{\bm{x}_f^i\}_{i=1}^{N_f}$ represents the set of $N_f$ residual points and $\{\bm{x}_u^i\}_{i=1}^{N_u}$ represents the set of $N_u$ training data points.  $w_{\mathcal{F}}>0$ and $w_u>0$  are the weights for residual and training data points, respectively, which are user-defined or tuned automatically~\citep{wang2021understanding}. Specifically, we have
\begin{itemize}
    \item  \texttt{PDE residual loss} in Eq.~\eqref{eq: definition MSE_f} (i.e., the first term in Eq.~\eqref{eq: PINN loss function})  enforces the NN solution $u_{\bmTheta}$ to obey the PDE. If the dimension of $\bm{x}$ is one, it is an ordinary differential equation (ODE);
    \item  \texttt{Data loss} in Eq.~\eqref{eq: definition MSE_u} (i.e., the second term in Eq.~\eqref{eq: PINN loss function}) includes the known initial/boundary conditions and/or experimental data, which must be satisfied by the NN solution $u_{\bmTheta}$.
\end{itemize}

The resulting optimization problem is to find the minimum of the loss function by optimizing the parameters, i.e., we seek to find
\begin{equation}\label{eq: optimization}
\bmTheta^*=\underset{\bmTheta }{\arg \min}\ J(\bmTheta).
\end{equation}
The solution to problem~\eqref{eq: optimization} can be approximated iteratively by one of the forms of gradient descent algorithm, such as Adam~\citep{kingma2014adam} and L-BFGS~\citep{byrd1995limited}.



\subsection{Self-scalable Tanh Activation Function} \label{subsec: Proposed Activation function}
The role of an activation function is to induce non-linearity into the NN model. It is well known that NN with a non-polynomial activation function can model any function~\citep{leshno1993multilayer}.  The nonlinear activation function performs the nonlinear transformation to the input data, making it capable of learning and performing more complex tasks. In general, the activation functions need to be differentiable so that the back-propagation is possible, where the gradients are supplied to update the weights and biases. 

In PINN, the activation function also plays a role in finding the derivatives (for example, $\frac{du_{\bmTheta}}{dx}$ to find \texttt{PDE residual loss}). As mentioned previously, non-smooth activation functions are not suitable for the PINN task~\citep{mishra2020estimates}. The tanh and its adaptive modification~\citep{jagtap2020adaptive, jagtap2020locally}  remain the most sought-after activation function for PINN ~\citep{jagtap2020conservative, jagtap2020extended, singh2019pi, zhang2021physics}. Eq.~\eqref{eq: tanh} gives the widely used tanh function~\citep{lecun2012efficient} as
\begin{equation}\label{eq: tanh}
    \textnormal{tanh}(x)=\frac{e^x-e^{-x}}{e^x+e^{-x}}.
\end{equation}
tanh has several advantages including continuity and differentiability providing smooth gradient over the entire support~\citep{lecun2012efficient}. However, in the context of PINN, there are two main issues for the widely used tanh-based activation functions: 

 $\bullet$ The first one is that PINN is more likely to have the problem of vanishing gradients due to the calculation of \texttt{PDE residual loss}.  The vanishing gradients are caused by the \textit{saturation} of the activation function, which has the following definition.
\begin{definition}[Saturation~\citep{gulcehre2016noisy}]
An activation function $\sigma(x)$ with derivative $\sigma'(x)$  is said to right (resp. left) saturate if its limit as $x\to +\infty$ (resp. $x\to -\infty$) is zero. An activation function is said to saturate (without qualification) if it both
left and right saturates.
\end{definition}
For example, $\mathcal{F}_{\bmTheta}$ at a particular $x_f^i$ can be calculated by computing the derivative (i.e., $\frac{du_{\bmTheta}}{dx}|_{x=x_f^i}$) for any given problem (refer to Section \ref{sec: PINN study}) by AD. This is done by using the chain rule of differentiation considering the NN weights and biases as constant. With the notation introduced in Section \ref{subsec: PINN} and some simplification of it for clarity, the expression for the required derivative can be roughly given as 
\begin{equation} \label{eq: chain rule AD}
    \frac{d u_{\bmTheta}}{d x}\Bigg|_{x=x_f^i} = \Bigg[ \frac{\partial \mathcal{L}_D}{\partial z_{D-1}}\times \frac{\partial z_{D-1}}{\partial \mathcal{L}_{D-1}} \times \frac{\partial \mathcal{L}_{D-1}}{\partial z_{D-2}}\times\cdots\times\frac{\partial z_1}{\partial \mathcal{L}_1}\times\frac{\partial \mathcal{L}_1}{\partial x}\Bigg]\Bigg|_{x=x_f^i},
\end{equation}
which is very similar to the back-propagation that includes the gradient of activation functions at each layer (i.e., $\frac{\partial z_{k}}{\partial \mathcal{L}_{k}}$). The derivative in Eq.~\eqref{eq: chain rule AD} is a part of \texttt{PDE residual loss}, where  the gradient of \texttt{PDE residual loss}  with respect to weights and biases of PINN (i.e., $\frac{\partial \textnormal{MSE}_{\mathcal{F}}}{\partial \bmTheta})$ is  calculated by back-propagation~\citep{lagaris1998artificial,Goodfellow-et-al-2016}. This makes PINN have a ``deeper" structure, to calculate the gradients of parameters, than standard neural networks. Thus, PINN is more likely to suffer from vanishing gradients caused by saturation~\citep{wang2021understanding}. It is theoretically analyzed by \cite{wang2021understanding} that vanishing gradients can lead to a concentration of training gradients around zero for optimizing the \texttt{PDE residual loss}. The problem gets even more troublesome when higher-order derivatives and larger dimensional equations are solved. 




 $\bullet$ The second issue is the unstable training due to the requirement of varied magnitudes of weights. The output from tanh (or the adaptive version) is bounded between -1 and 1. When using tanh for regression problems, it is well known that the last layer should be linear to allow the output to take any value and not constrained by the activation function~\citep{mathew2020deep}. When the output is not normalized, it can lead to unstable training~\citep{ranjan2020understanding,bishop1995neural}. 
As already introduced in Section~\ref{sec: intro}, the output from PINN can not be normalized since  the magnitude is not known with the PDE and initial/boundary conditions. This implies, for an accurate model with tanh activation, the penultimate layer has an output between -1 and 1, but the final layer's output should be of the appropriate range (possibly high). This further implies that the last layer weights should be higher in magnitude than the rest of the layers. This complicates the training~\citep{bishop1995neural} for both tanh function and its adaptive modification. It is preferred that the model gradually scales into different orders of magnitude than just using the last layer for scaling up. 

Both of the mentioned issues are more pronounced when the outputs are of higher orders of magnitude. Therefore, to address these two issues
\begin{itemize}
    \item  First, there is a need for higher gradients and non-saturation behaviour for activation functions;
    \item Second, the activation function should be unbounded.
\end{itemize}
 Accordingly, the proposed Self-scalable tanh (Stan) activation function for $i$th neuron in $k$th layer ($k=1,2,\dots,D-1; \ i=1,2,\dots,N_k$) is proposed as
\begin{equation}\label{eq: proposed activation}
    \sigma^i_k(x)=\textnormal{tanh}(x)+\beta^i_k x\cdot\textnormal{tanh}(x),
\end{equation}
where  $\beta^i_k$ is the neuron-wise parameter that should be optimized. The second term, namely, $\beta^i_k x\cdot \textnormal{tanh}(x)$, is the self-scaling term that can help to better map different orders of magnitude of input and output when allowing the gradients to flow without vanishing. Figure~\ref{fig: stan function visual} shows the visualization of Stan activation and its gradient. \begin{figure}[!htb]
\centering
\subfloat[Stan activation function]{%
\resizebox*{7.1cm}{!}{\includegraphics{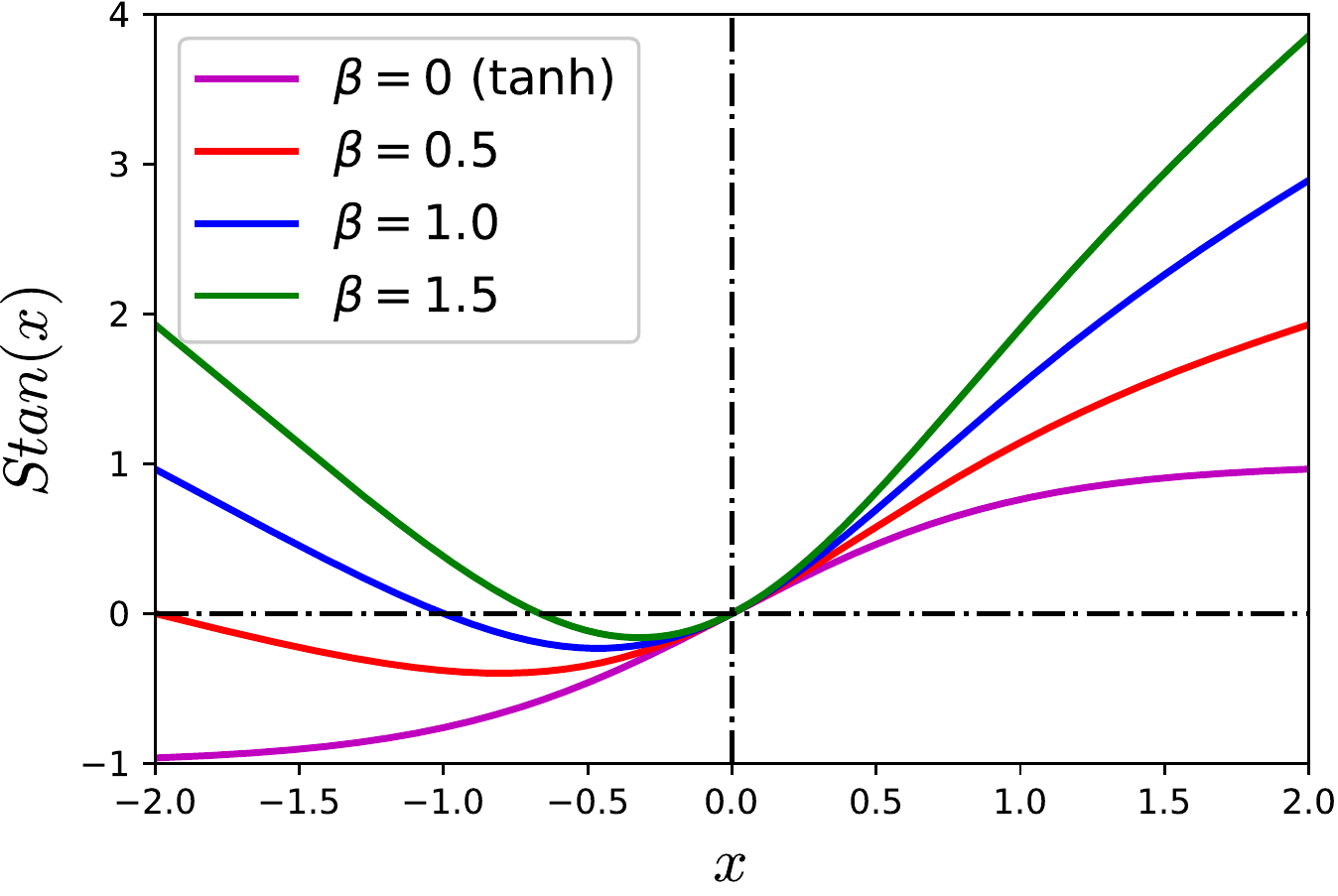}} \label{subfig: stan}  }\hspace{5pt}
\subfloat[Stan activation function gradient]{%
\resizebox*{7cm}{!}{\includegraphics{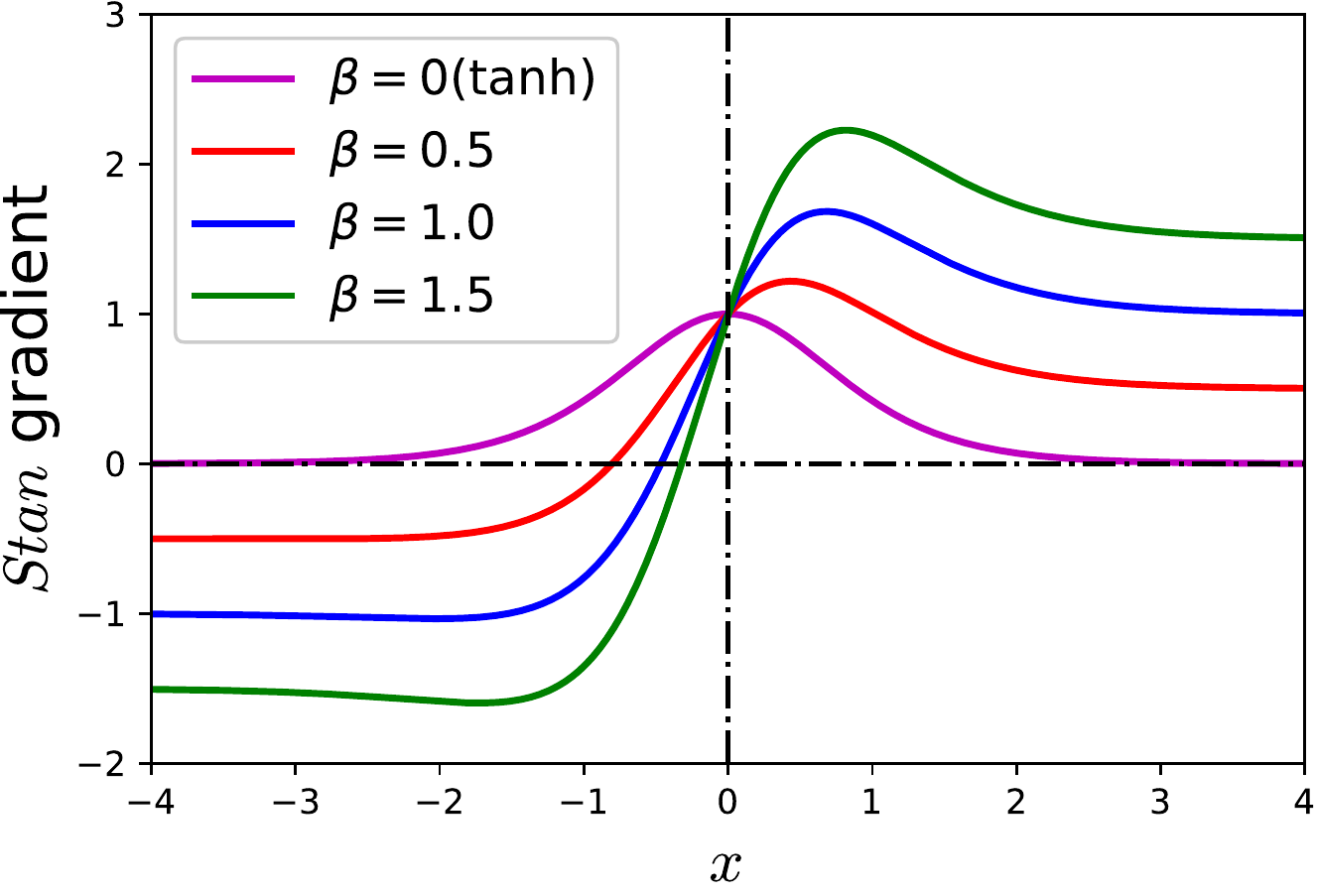}}  \label{subfig: stan gradient}}\hspace{5pt}\\
\caption{Visualization of the proposed Stan activation function in comparison with tanh. At $x=0$, the function value is 0 and the gradient is 1. The point of \textit{zero gradient} and \textit{maximum gradient} are dependent on $\beta$.} \label{fig: stan function visual}
\end{figure}
From Figure~\ref{subfig: stan gradient}, our proposed Stan activation function does not have the saturation problem, which can be formally quantified in the following proposition.
\begin{proposition}
If $\beta^i_k\neq 0$, the proposed activation function in~\eqref{eq: proposed activation} is not saturation.
\end{proposition}
\begin{proof}
The first-order derivative of  $\sigma^i_k(x)$ is 
\begin{equation*}
    (\sigma_k^{i})'(x) = (1+\beta^i_k x)\Big(1-\textnormal{tanh}^2(x)\Big) + \beta^i_k\textnormal{tanh}(x).
\end{equation*}
Therefore, we have
\begin{equation*}
\begin{aligned}
     \lim_{x \to +\infty} (\sigma^i_k)'(x) &= \beta^i_k,\\
       \lim_{x \to -\infty} (\sigma^i_k)'(x)& = -\beta^i_k.
\end{aligned}
\end{equation*}
\end{proof}

 

The proposed Stan gives additional $\sum_{k=1}^{D-1}N_k$ parameters to be optimized. In this case, we have $\tilde{\bmTheta}=\{\bm{W}^k,\bm{b}^k\}_{k=1}^D\bigcup\{\bm{\beta}^k\}_{k=1}^{D-1}$ are trainable parameters,~where $\bm{\beta}^k=(\beta_k^1,\beta_k^2,\dots,\beta_k^{N_k})^{\top}$.  The parameter $\tilde{\bmTheta}$ is updated as
\begin{equation} \label{eq: update rule}
    \tilde{\bmTheta}_{m+1} = \tilde{\bmTheta}_m -\eta_m\nabla J(\tilde{\bmTheta}_m),
\end{equation}
where $\eta_m\in (0,1]$ is the learning rate at $m$th iteration. The proposed activation function-based PINN algorithm is summarized in Algorithm~\ref{alg: PINN}. \begin{algorithm}[!htp]
\caption{Self-scalable tanh activation function-based PINN algorithm}  \label{alg: PINN} 
\textbf{Input:}  Training data: $u_{\tilde{\bmTheta}}$ network $\{\bm{x}_u^i\}_{i=1}^{N_u}$;  Residual training points: $\{\bm{x}_f^i\}_{i=1}^{N_f}$
\begin{algorithmic}
\State  \textbf{Step 1:} Construct neural network $u_{\tilde{\bmTheta}}$ with random initialization of parameters $\tilde{\bmTheta}$.
\State \textbf{Step 2:} Construct the residual $\mathcal{F}_{\tilde{\bmTheta}}$ of neural network  by substituting surrogate $u_{\tilde{\bmTheta}}$ into the governing equation using automatic differentiation and other arithmetic operations.
\State \textbf{Step 3:} Specification of loss function:
\begin{equation}\label{eq: final loss function}
    J(\tilde{\bmTheta}) =\frac{w_{\mathcal{F}}}{N_f}\sum_{i=1}^{N_f}\left| \mathcal{F}_{\tilde{\bmTheta}}(\bm{x}_f^i) \right|^2 + \frac{w_u}{N_u}\sum_{i=1}^{N_u}\left| u^i - u_{\tilde{\bmTheta}}(\bm{x}_u^i)\right|^2.
\end{equation}
\State  \textbf{Step 4:}  Find the best parameters $\tilde{\bmTheta}^*$ using a suitable optimization method for minimizing the loss function in~\eqref{eq: final loss function} as
\begin{equation*}
    \tilde{\bmTheta}^*=\underset{\tilde{\bmTheta}}{\arg \min}\ J(\tilde{\bmTheta}).
\end{equation*}
\end{algorithmic}
\textbf{Output:} $u_{\tilde{\bmTheta}^*}$ network 
\end{algorithm}

 We now provide a theoretical result regarding the PINN with the proposed Stan. The following theorem states that a gradient descent algorithm minimizing our objective function in Eq.~\eqref{eq: final loss function} with a L2 regularization term does not converge to a spurious stationary point  given appropriate initialization and learning rates. Now, let $J_{\gamma}(\tilde{\bmTheta})\coloneqq J_{\gamma}(\tilde{\bmTheta})+\gamma\|B\|_2^2$ (i.e., the objective function in Eq.~\eqref{eq: final loss function} with the L2 regularization) where $B\coloneqq \{\bm{\beta}_k\}_{k=1}^{D-1}$, $\gamma>0$ is a fixed hyper-parameter that can be defined by users. 
\begin{theorem}[No Spurious Stationary Points] \label{thm: no spurious}
Let $\{\tilde{\bmTheta}_m\}_{m\geq 0}$ be a sequence generated by a gradient descent algorithm based on the update rule in Eq.~\eqref{eq: update rule}. Assume that $J_{\gamma}(\tilde{\bmTheta}_0)<J_{\gamma}(\bm{0})$, $J_{\gamma}$ is differentiable, and that for each $i\in\{1,\dots,N_f\}$, there exist a differentiable function $\phi^j$ and input  $\rho^j$ such that $\left|\mathcal{F}(\bm{x}_f^j) \right|^2=\phi^j(u_{\tilde{\bmTheta}}(\rho^j))$, another differentiable function $\pi^i_k$ so that $\sigma^i_k(x)=\pi^i_k(\beta^i_k x)$ for all $i,k$.  Suppose one of the following conditions holds:
	\begin{enumerate}[label=(\roman*)]
	\item  $\nabla J_{\gamma}(\tilde{\bmTheta})$  is Lipschitz continuous with Lipschitz constant $C>0$. That is, $\|\nabla J_{\gamma}(\tilde{\bmTheta})-\nabla J_{\gamma}(\tilde{\bmTheta}')\|_2\leq C\|\tilde{\bmTheta}-\tilde{\bmTheta}'\|_2$ for all $\tilde{\bmTheta}, \tilde{\bmTheta}'$ in its domain. And $\varepsilon \leq \eta_m\leq \frac{2-\varepsilon}{C}$, where  $\varepsilon>0$ is a fixed positive number;
	\item   $\nabla J_{\gamma}(\tilde{\bmTheta})$  is Lipschitz continuous, $\eta_m \to 0$, and $\sum_{m=1}^{+\infty}\eta_m=+\infty$;
    \item  The learning rate $\eta_m$ is chosen by the minimization rule, the limited minimization rule, the Armjio rule, or the Goldstein rule~\citep{bertsekas1997nonlinear}. 
\end{enumerate}
Then, no limit point of $\{\tilde{\bmTheta}_m\}_{m\geq 0}$ is a spurious stationary point.
\end{theorem}
\begin{proof}
See proof in Appendix~\ref{appendix: proofs}.
\end{proof}
The initial condition $J_{\gamma}(\tilde{\bmTheta}_0)<J_{\gamma}(\bm{0})$ means that the initial value $J_{\gamma}(\tilde{\bmTheta}_0)$ needs to be less than that of a constant network. This analysis is not possible without introducing the $\{\bm{\beta}_k\}_{k=1}^{D-1}$ in the proposed activation function.  The above Theorem~\ref{thm: no spurious} shows that if the algorithm converges to a stationary point, it converges to global optimal solutions under three independent conditions. The conditions \textit{(i)}, \textit{(ii)}, and \textit{(iii)} correspond to the cases of the constant learning rate, diminishing learning rate,
and adaptive learning rate, respectively. For our experiments in Sections~\ref{sec: numerical study} and~\ref{sec: PINN study}, the constant learning rate is adopted. The training loss from these experiments, which usually converge to zero, illustrates that the algorithm can escape the problem of Spurious Local Minima~\citep{safran2018spurious} in neural networks. 







\section{Numerical Study} \label{sec: numerical study}
In Section~\ref{subsec: function approximation}, nonlinear smooth and discontinuous functions used in~\cite{jagtap2020adaptive} are used to verify the effectiveness of the proposed activation function, namely, Stan. The sensitivity analysis on the initialization of $\beta^i_k$s is presented in Section~\ref{subsec: sensitivity analysis}.  In all analyses, tanh and Neuron-wise Locally Adaptive Activation Function (N-LAAF)~\cite{jagtap2020locally} are selected as benchmarks for comparison with the proposed activation function, which are state-of-the-art activation functions for PINN.  N-LAAF is an adaptive modification of the tanh function by introducing an adaptive parameter $a^i_k$ for every neuron as $\textnormal{tanh}(a^i_kx)$. For this numerical study, the same setup in~\citep{jagtap2020adaptive} is used, where the neural network architecture  consists of four hidden layers with 50 neurons in each layer, and Adam optimizer (learning rate 0.0008) is utilized for training. All the results are repeated ten times with different initializations of NN weights and biases to obtain the average performance. The codes of all the experiments are implemented in Python 3 with Pytorch 1.9.0 package. The GPU used in this work is an NVIDIA 2080 Ti.
\begin{figure}[!htb]\vspace{-0.0cm}
\centering
\subfloat[Empirical convergence of training loss]{%
\resizebox*{7cm}{!}{\includegraphics{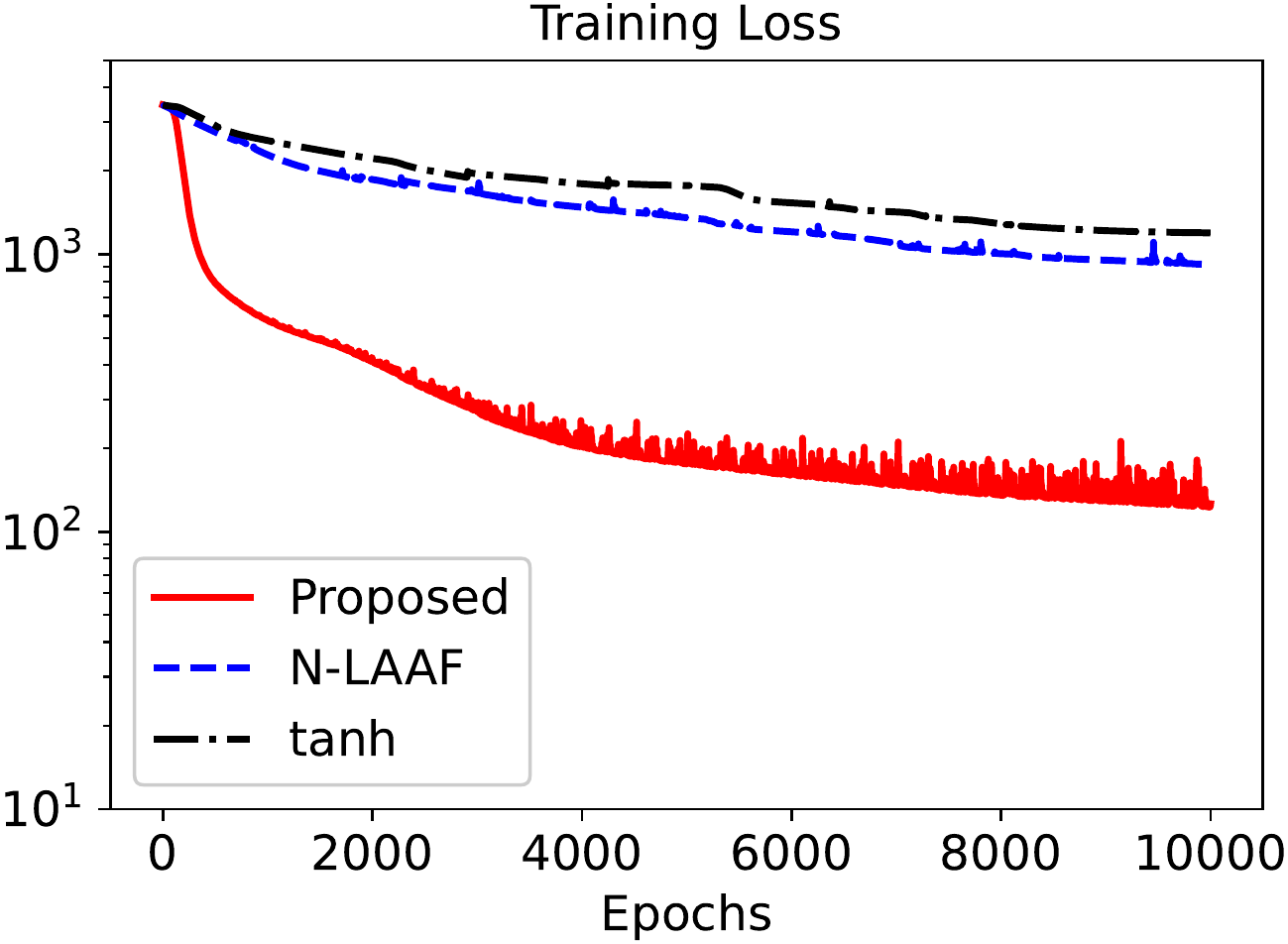}} \label{subfig: smooth function-training loss}  }\hspace{5pt}
\subfloat[Empirical convergence of $\beta^i_k$ in one of the neurons]{%
\resizebox*{7.5cm}{!}{\includegraphics{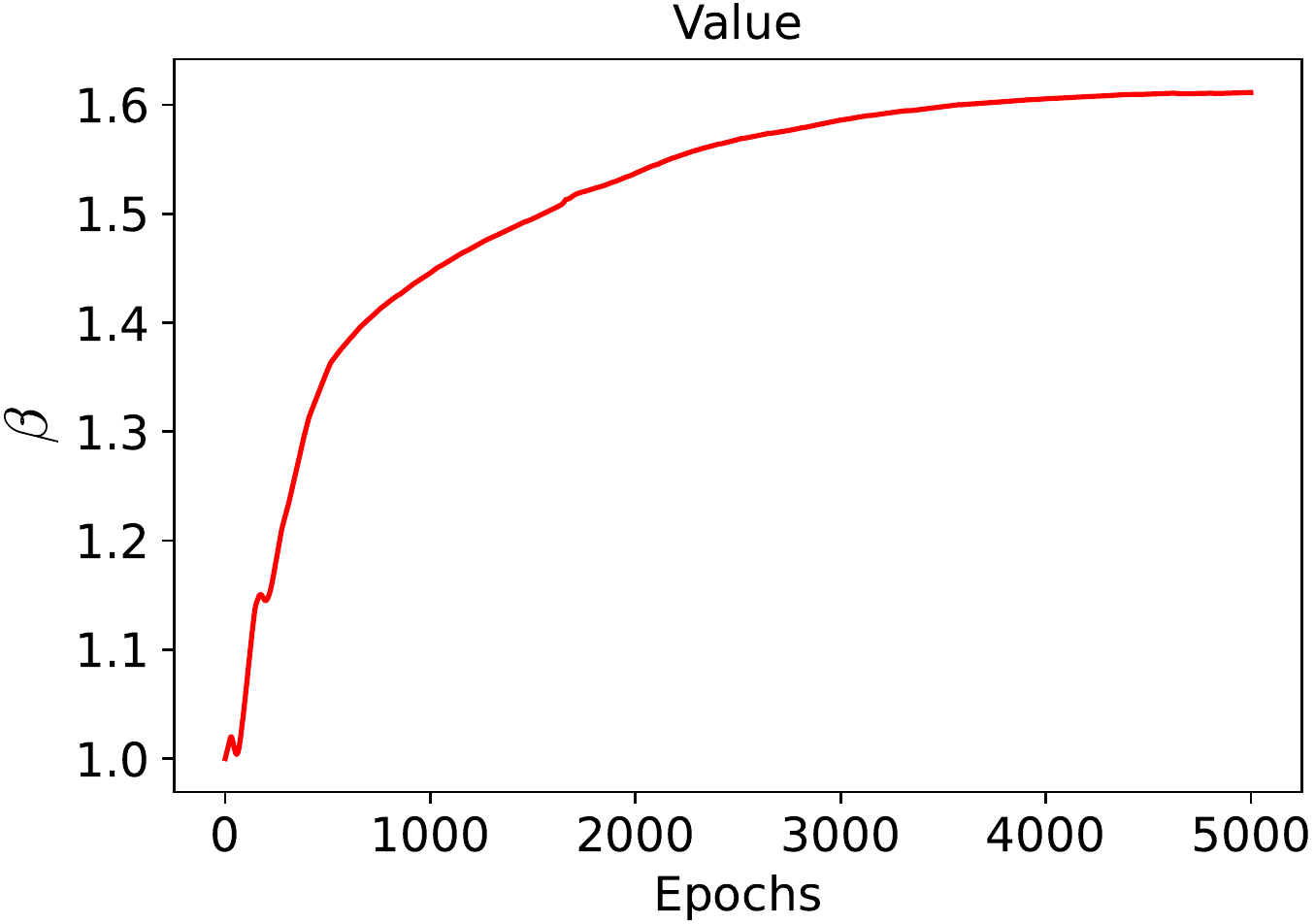}}   \label{subfig: smooth function-beta} }\\
\subfloat[Neural network solution of the proposed method]{%
\resizebox*{5cm}{!}{\includegraphics{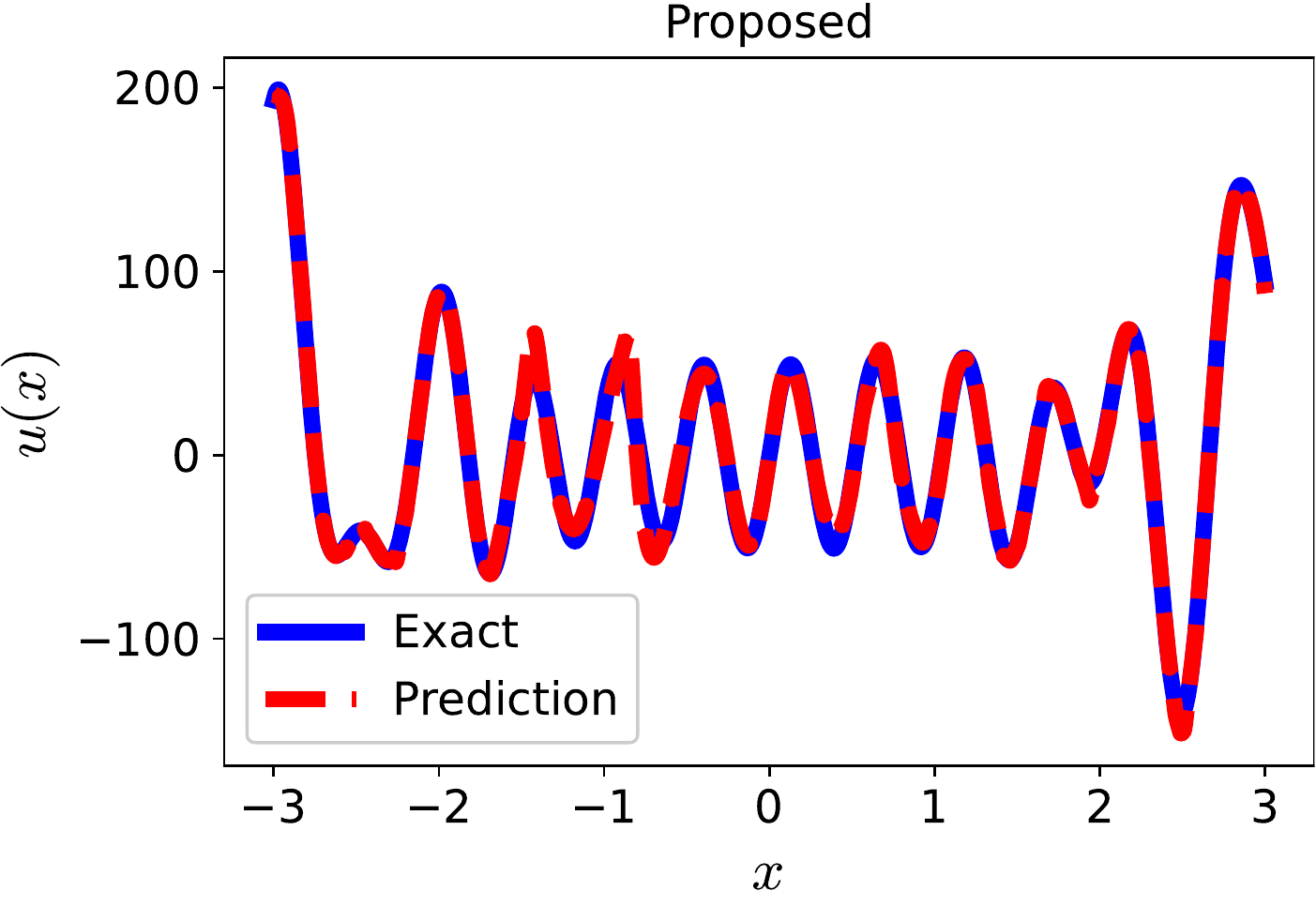}} \label{subfig: smooth function-proposed} }
\subfloat[Neural network solution of N-LAAF]{%
\resizebox*{5cm}{!}{\includegraphics{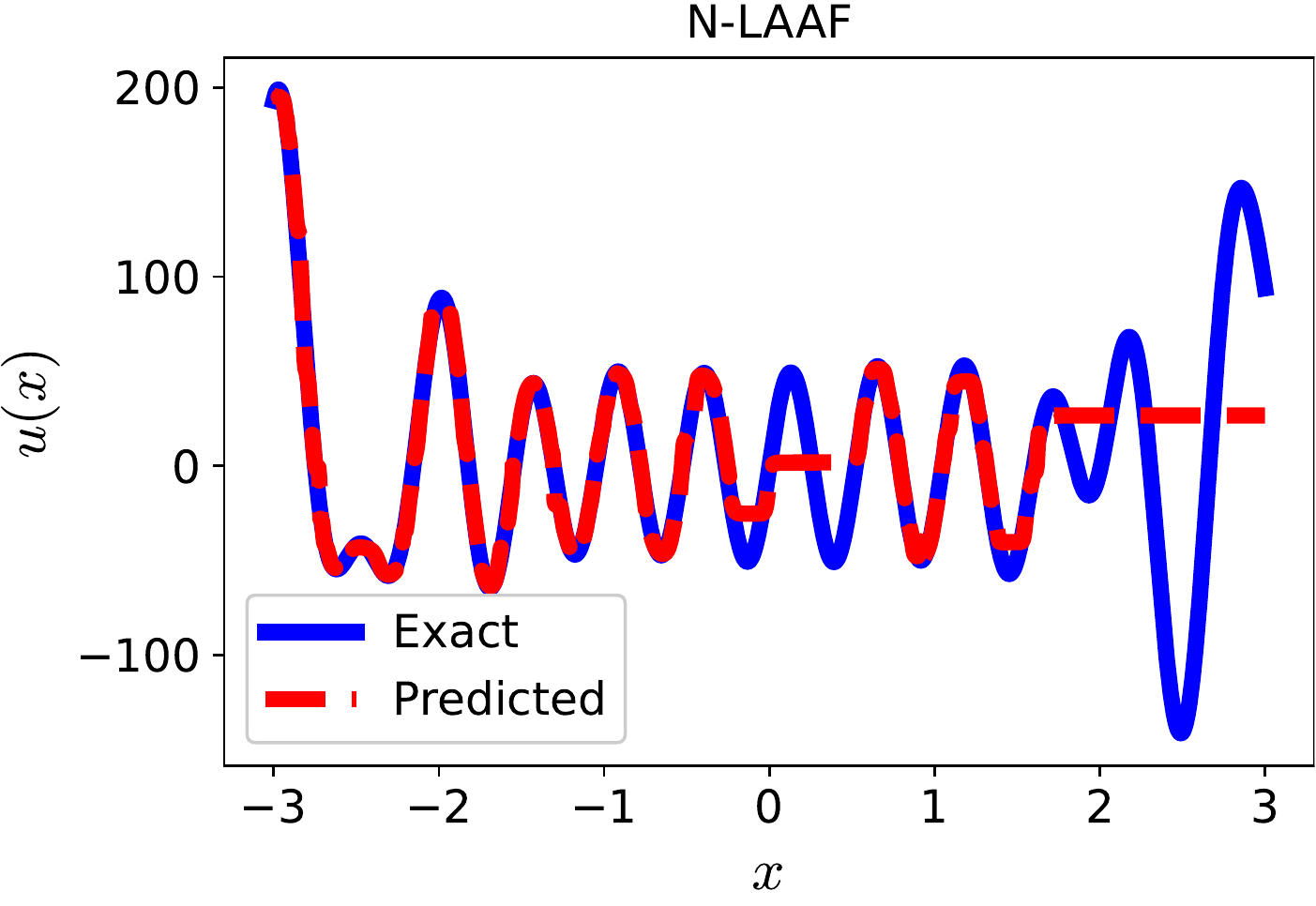}} \label{subfig: smooth function-N-LAAF} }
\subfloat[Neural network solution of tanh]{%
\resizebox*{5cm}{!}{\includegraphics{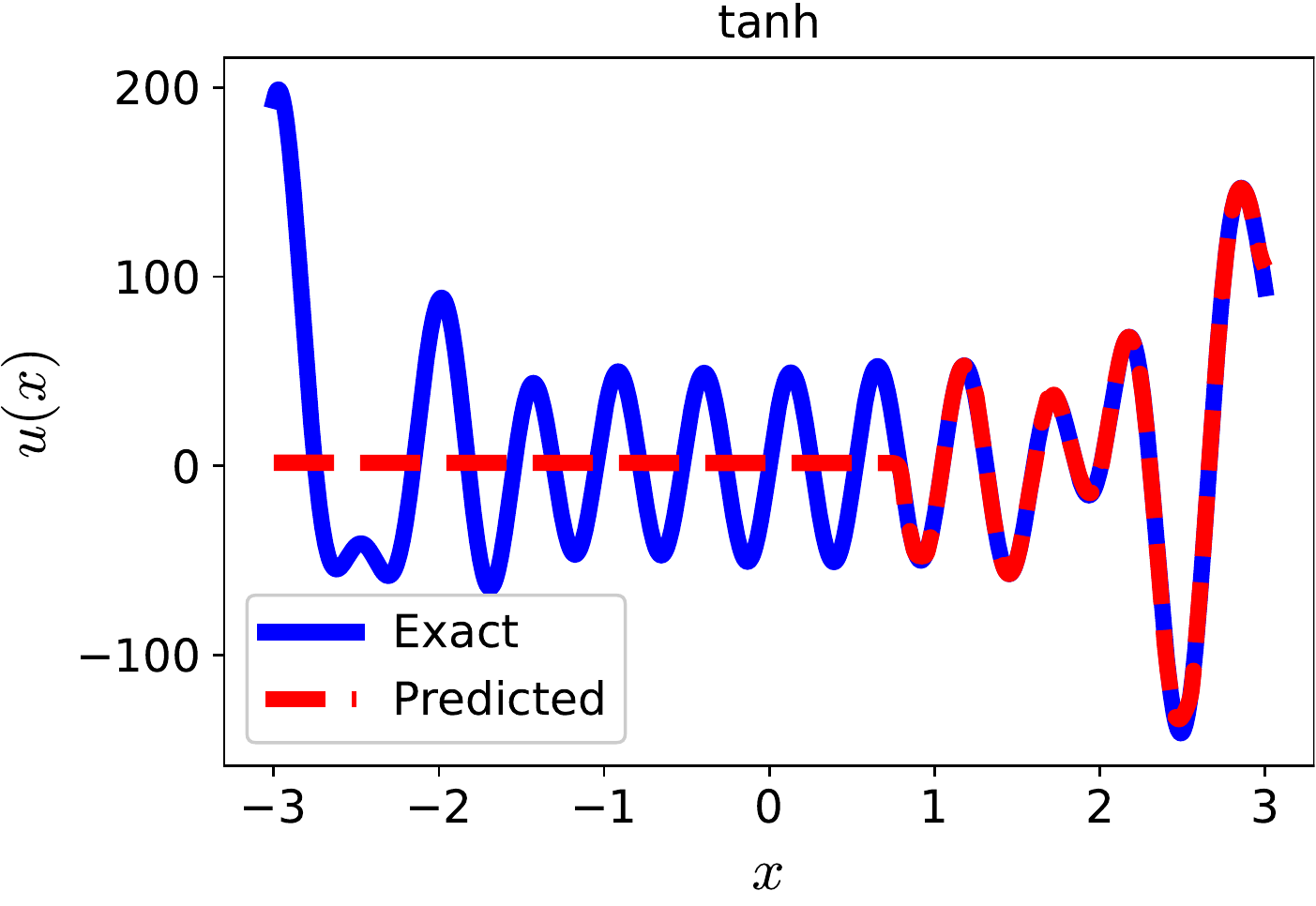}} \label{subfig: smooth function-tanh} }
\caption{Results of the smooth function: (a) Empirical convergence of training loss; (b) Empirical convergence of $\beta^i_k$; (c) Neural network solution of the proposed method; (d) Neural network solution of N-LAAF; (e) Neural network solution of tanh.} \label{fig: performance of smooth function}
\end{figure}
\subsection{Neural Network Approximation of Functions} \label{subsec: function approximation}
In this subsection, a standard NN (without the physics-informed part) is employed to approximate smooth and discontinuous functions. Specifically, the smooth and discontinuous functions defined by~\cite{jagtap2020adaptive} are scaled with a constant factor given as
\begin{equation} \label{eq: smooth function expression}
    u(x) = 50 \times  \Big[(x^3 - x) \frac{\sin(7x)}{7} + \sin(12x) \Big], \,\, x \in [-3,3],
\end{equation}
and
\begin{align} \label{eq: discontinuous function expression}
 u(x) = \Bigg\{
 \begin{split}
     40\sin(6x),\,\, -4 &\leq x \leq 0\\
     200 + 20x\cos(12x),\,\, 0&<x \leq 3.75. \\
 \end{split}    
\end{align}
300 randomly sampled points with their corresponding $u(x)$ values are used to train the model and 1,000 evenly spaced points for testing. The output data is not normalized to test the effectiveness of the proposed activation function. The parameters $\beta^i_k$s  are initialized with the value 1. The loss function for this case is the mean squared error of prediction on the training data as given below
\begin{equation*} \label{eq: regression loss}
     \textnormal{MSE}_u  =\frac{1}{300}\sum_{i=1}^{300}\left| u(x^i_u) - u_{\tilde{\bmTheta}}(x^i_u)\right|^2.
\end{equation*}
   
\begin{figure}[!htb]\vspace{-0.0cm}
\centering
\subfloat[Empirical convergence of training loss]{%
\resizebox*{7cm}{!}{\includegraphics{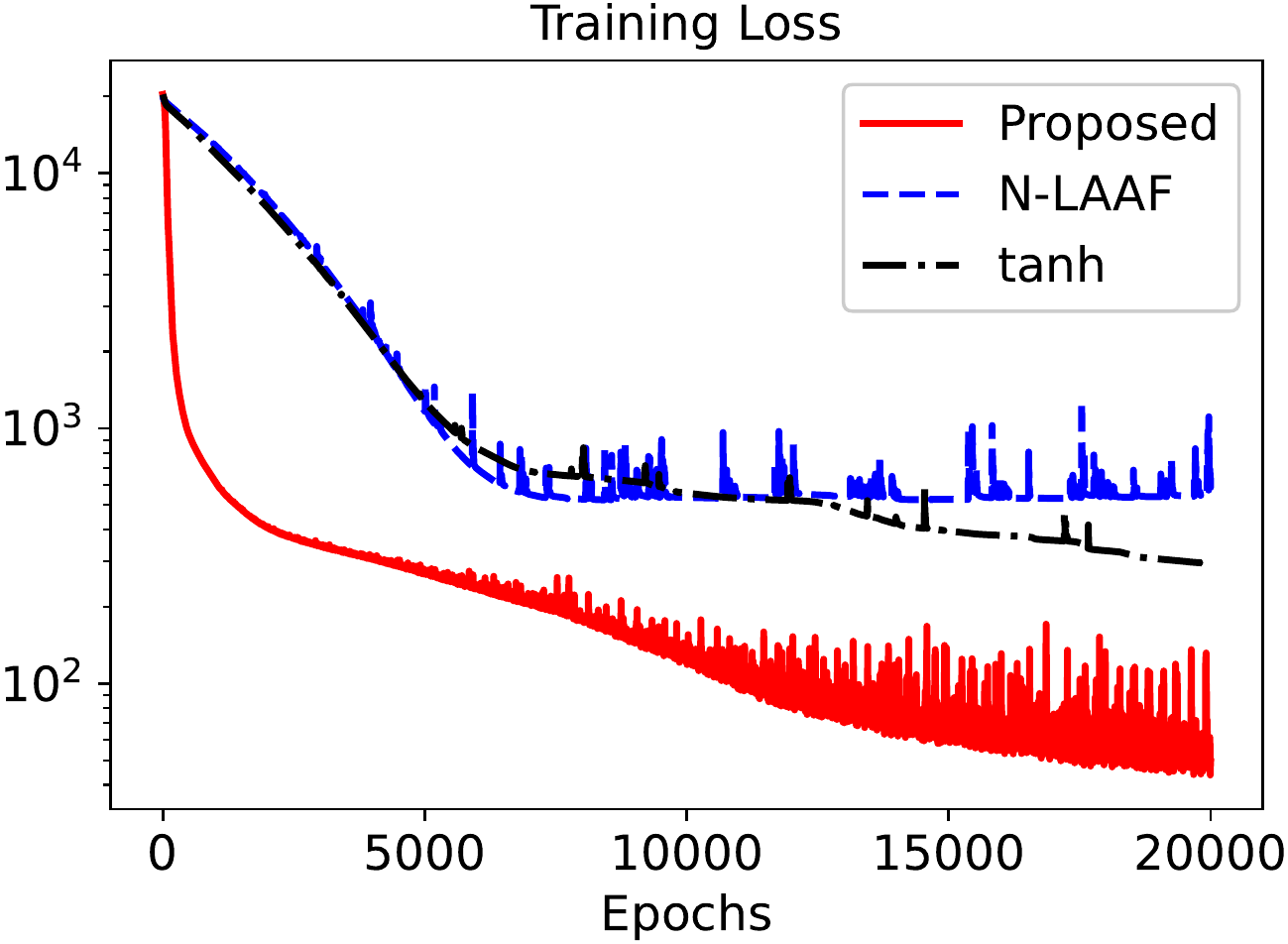}} \label{subfig: discontinuous function-training loss} }\hspace{5pt}
\subfloat[Empirical convergence of $\beta^i_k$ in one of the neurons]{%
\resizebox*{7.6cm}{!}{\includegraphics{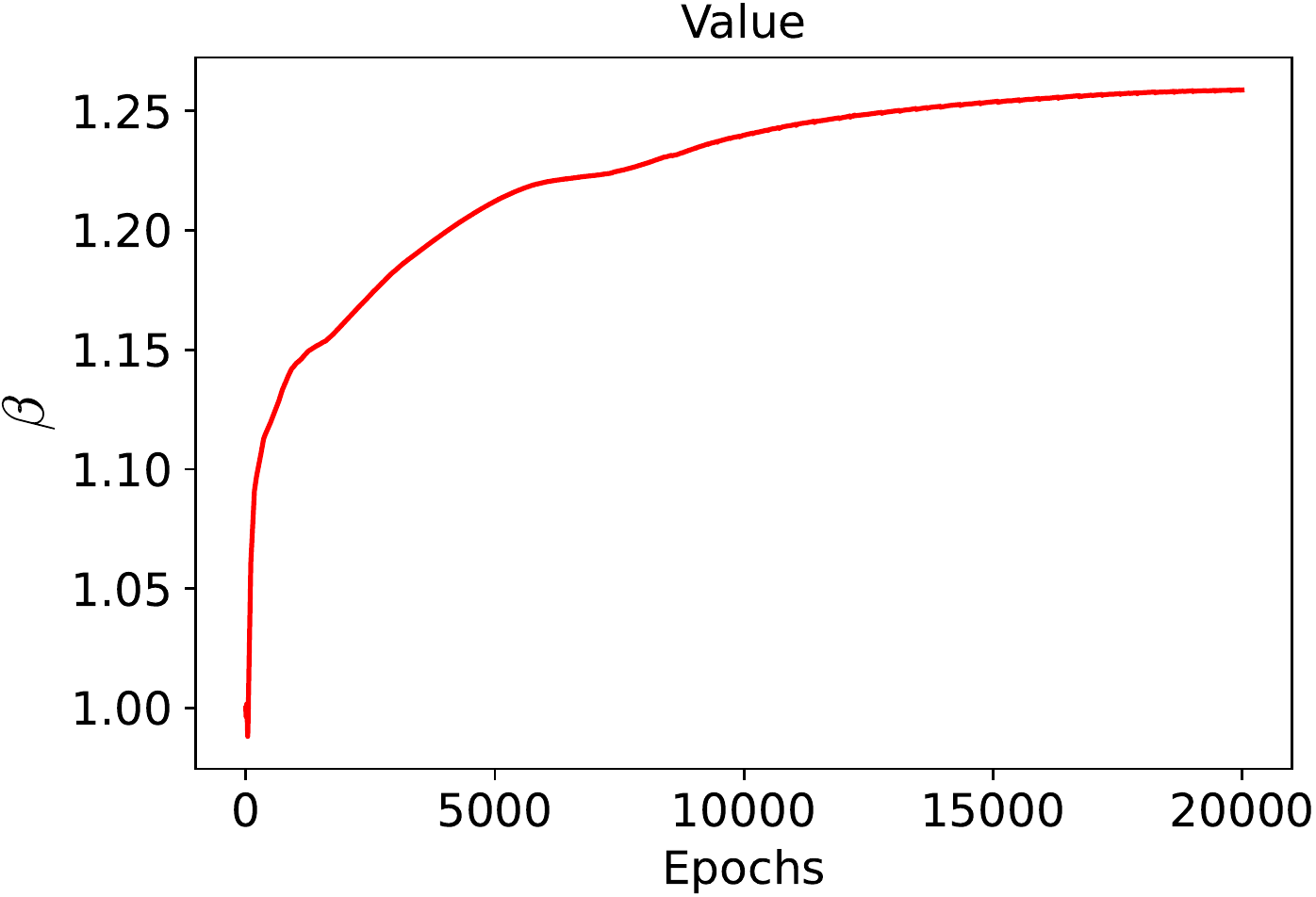}} \label{subfig: discontinuous function-beta}  }\\
\subfloat[Neural network solution of the proposed method]{%
\resizebox*{5cm}{!}{\includegraphics{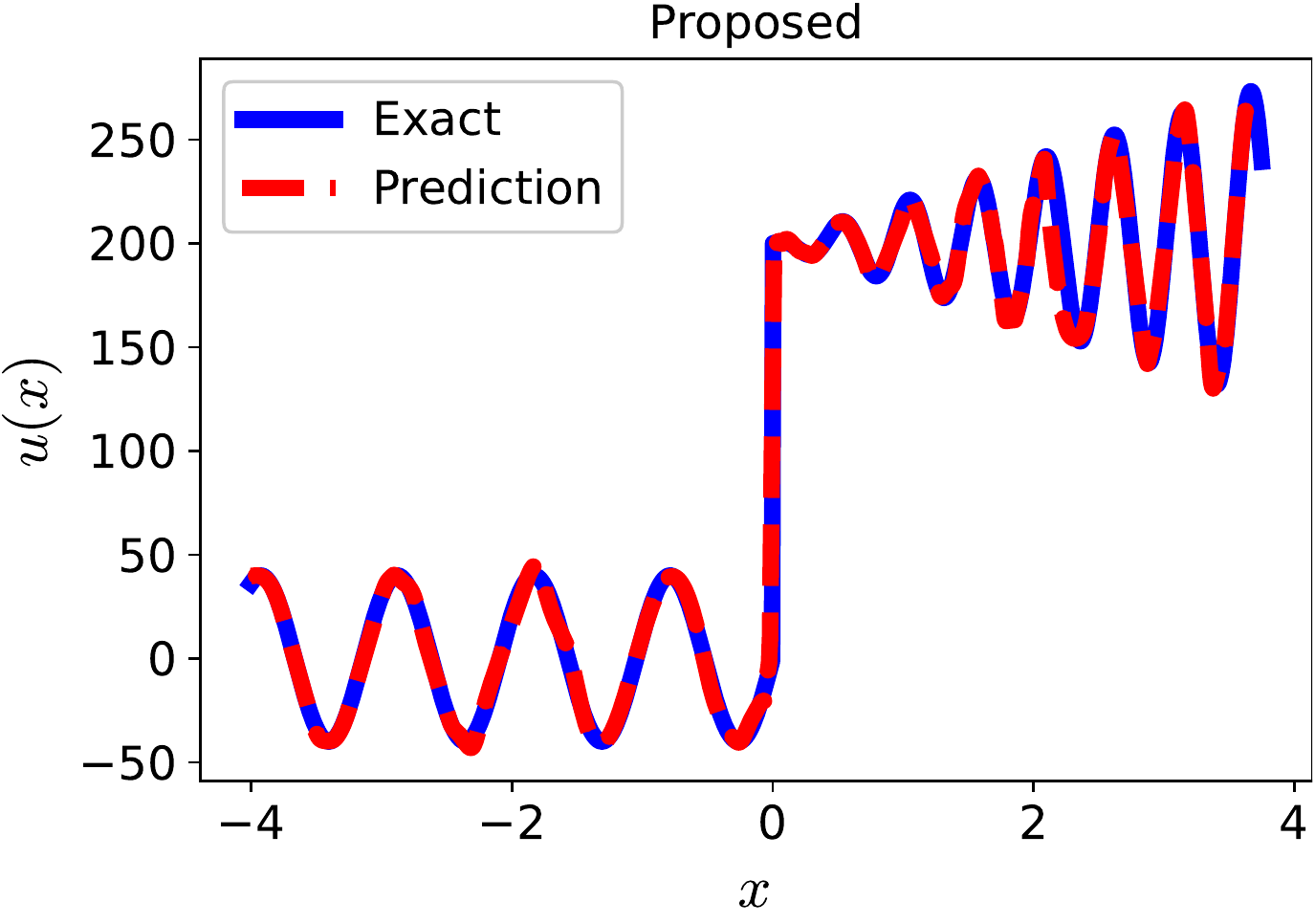}}  \label{subfig: discontinuous function-proposed} }
\subfloat[Neural network solution of N-LAAF]{%
\resizebox*{5cm}{!}{\includegraphics{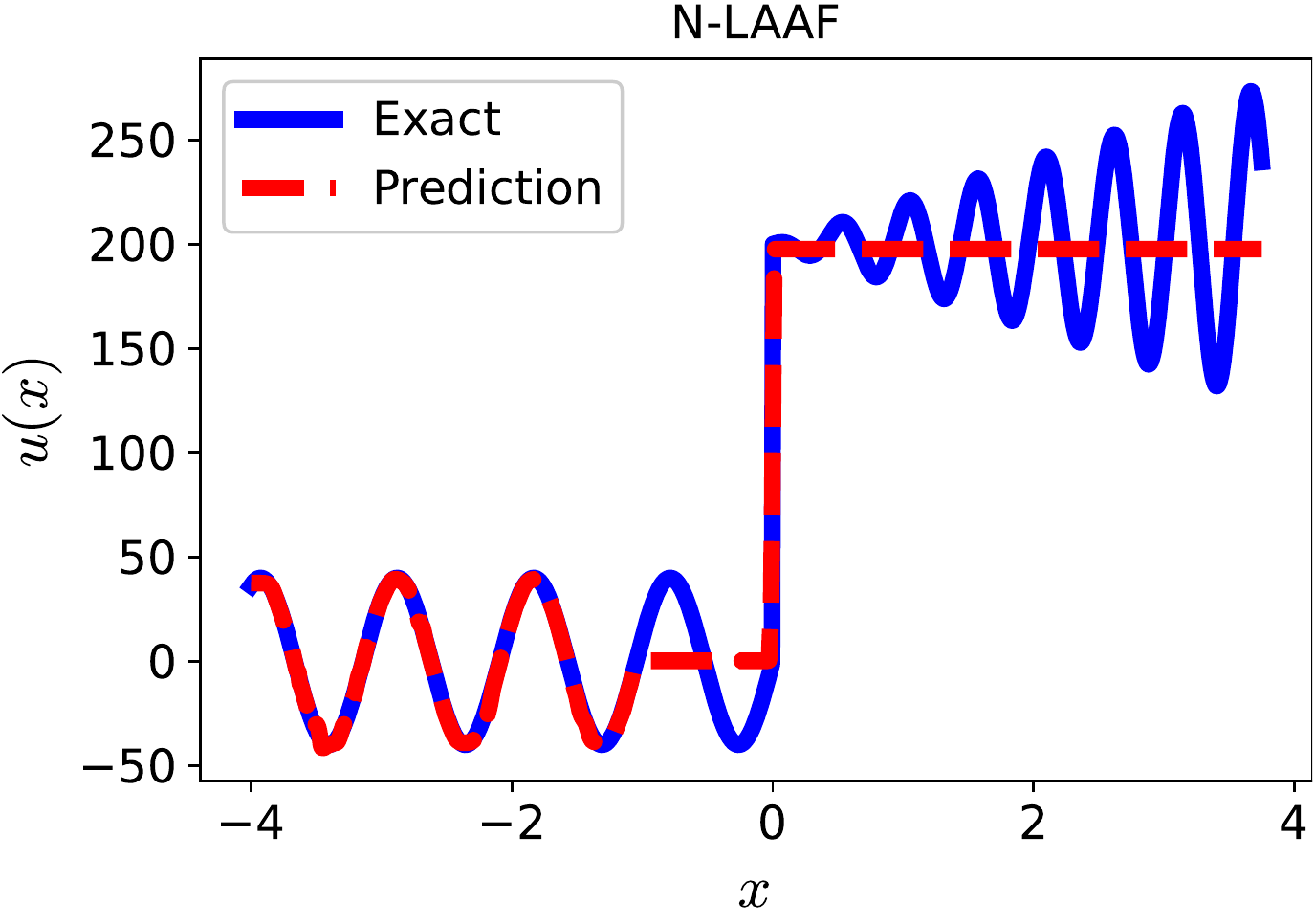}} \label{subfig: discontinuous function-N-LAAF}  }
\subfloat[Neural network solution of tanh]{%
\resizebox*{5cm}{!}{\includegraphics{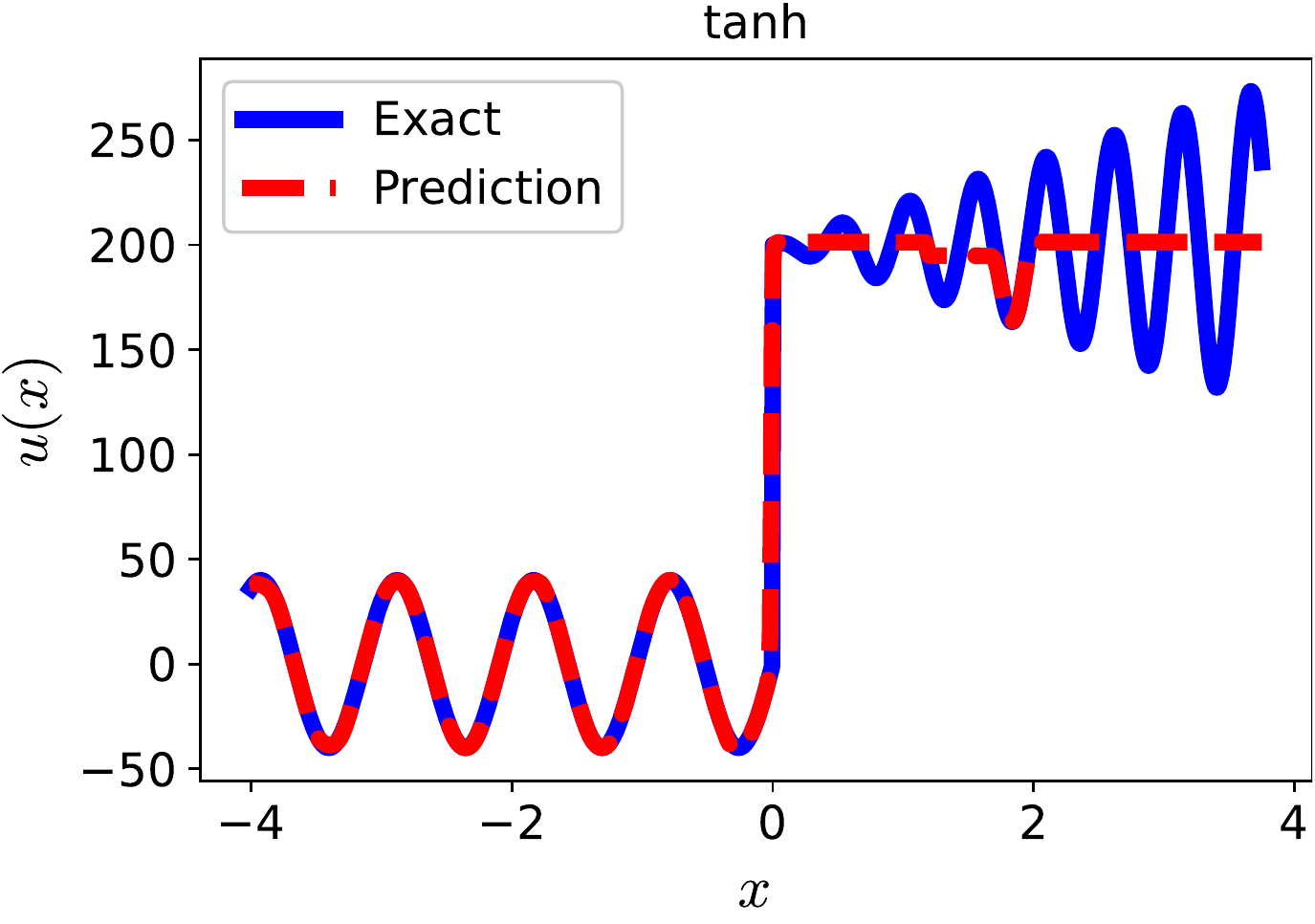}} \label{subfig: discontinuous function-tanh} }
\caption{Results of the discontinuous function: (a) Empirical convergence of training loss; (b) Empirical convergence of $\beta^i_k$; (c) Neural network solution of the proposed method; (d) Neural network solution of N-LAAF; (e) Neural network solution of tanh.}
\label{fig: performance of discontinuous function}
\end{figure}
 The results are summarized in Figures \ref{fig: performance of smooth function} and \ref{fig: performance of discontinuous function}.  Figure \ref{subfig: smooth function-training loss} shows the training performance of the smooth function with the proposed activation compared to benchmark methods. After 10,000 epochs, the proposed method shows significantly better training loss than other benchmark methods.  The convergence of $\beta_k^i$ in one of the randomly chosen neurons is shown in Figure \ref{subfig: smooth function-beta}. For that particular $\beta^i_k$, the value converges to around $1.6$ with almost a similar rate as the training loss.
 
 Figures \ref{subfig: smooth function-proposed}, \ref{subfig: smooth function-N-LAAF} and \ref{subfig: smooth function-tanh} show the model prediction of different methods in the  testing phase.  The proposed method is the only one that can perfectly fit the function in Eq.~\eqref{eq: smooth function expression}, while other methods have a poor fitting.  Figure \ref{subfig: discontinuous function-training loss} similarly shows the training performance of the discontinuous function for 20,000 epochs. The convergence of $\beta_k^i$ (to around $1.25$) in one of the neurons is also shown in Figure \ref{subfig: discontinuous function-beta}. Figures \ref{subfig: discontinuous function-proposed}, \ref{subfig: discontinuous function-N-LAAF} and \ref{subfig: discontinuous function-tanh} show that our proposed method performs best for model prediction  among all methods.  

\begin{table}[htbp!] 
\centering
\caption{Test performance of  function approximation for different methods in terms of MSE and RE.}\vspace{+0.4cm}
\begin{tabular}{lrrrrrr} 
\toprule
              & \multicolumn{2}{c}{tanh}                         & \multicolumn{2}{c}{N-LAAF}                     & \multicolumn{2}{c}{Proposed}                      \\ 
\cmidrule(lr){2-3}\cmidrule(lr){4-5}\cmidrule(lr){6-7}
              & \multicolumn{1}{c}{MSE} & \multicolumn{1}{c}{RE} & \multicolumn{1}{c}{MSE} & \multicolumn{1}{c}{RE} & \multicolumn{1}{c}{MSE} & \multicolumn{1}{c}{RE}  \\ 
\midrule
Eq.~\eqref{eq: smooth function expression}        & 1369.60                 & 0.6249                 & 988.18                  & 0.5308                 & \textbf{185.46}                  & \textbf{0.2299}                  \\ 
\midrule
Eq.~\eqref{eq: discontinuous function expression} & 47.77                   & 0.1930                 & 37.32                   & 0.1706                 & \textbf{23.97}                   & \textbf{0.1367}                  \\ 
\bottomrule
\end{tabular}\label{tab: test loss summary 1}
\end{table}
In addition, the test performance  for different methods in terms of MSE and RE (relative error)~\citep{raissi2019physics, jagtap2020adaptive} is summarized in Table~\ref{tab: test loss summary 1}. The proposed method achieves the best performance for both functions in Eqs.~\eqref{eq: smooth function expression} and \eqref{eq: discontinuous function expression} in terms of MSE and RE. The better performance in these two cases can be attributed mainly to the second issue discussed in Section~\ref{subsec: Proposed Activation function}.

\subsection{Sensitivity Analysis}\label{subsec: sensitivity analysis}
This study is focused on the activation function and the proposed activation function has trainable parameters $\beta_k^i$s. For the case studies in~\ref{subsec: function approximation}, these values are initialized as 1. In this subsection, the sensitivity analysis is conducted concentrating on the initialization of these $\beta_k^i$s. Specifically, for both the functions used in Section~\ref{subsec: function approximation}, all $\beta_k^i$s are initialized from ten values, which are equally spaced in the range $[0.25,1.15]$. \begin{figure}[!htb]
\centering
\subfloat[Smooth Function]{%
\resizebox*{7cm}{!}{\includegraphics{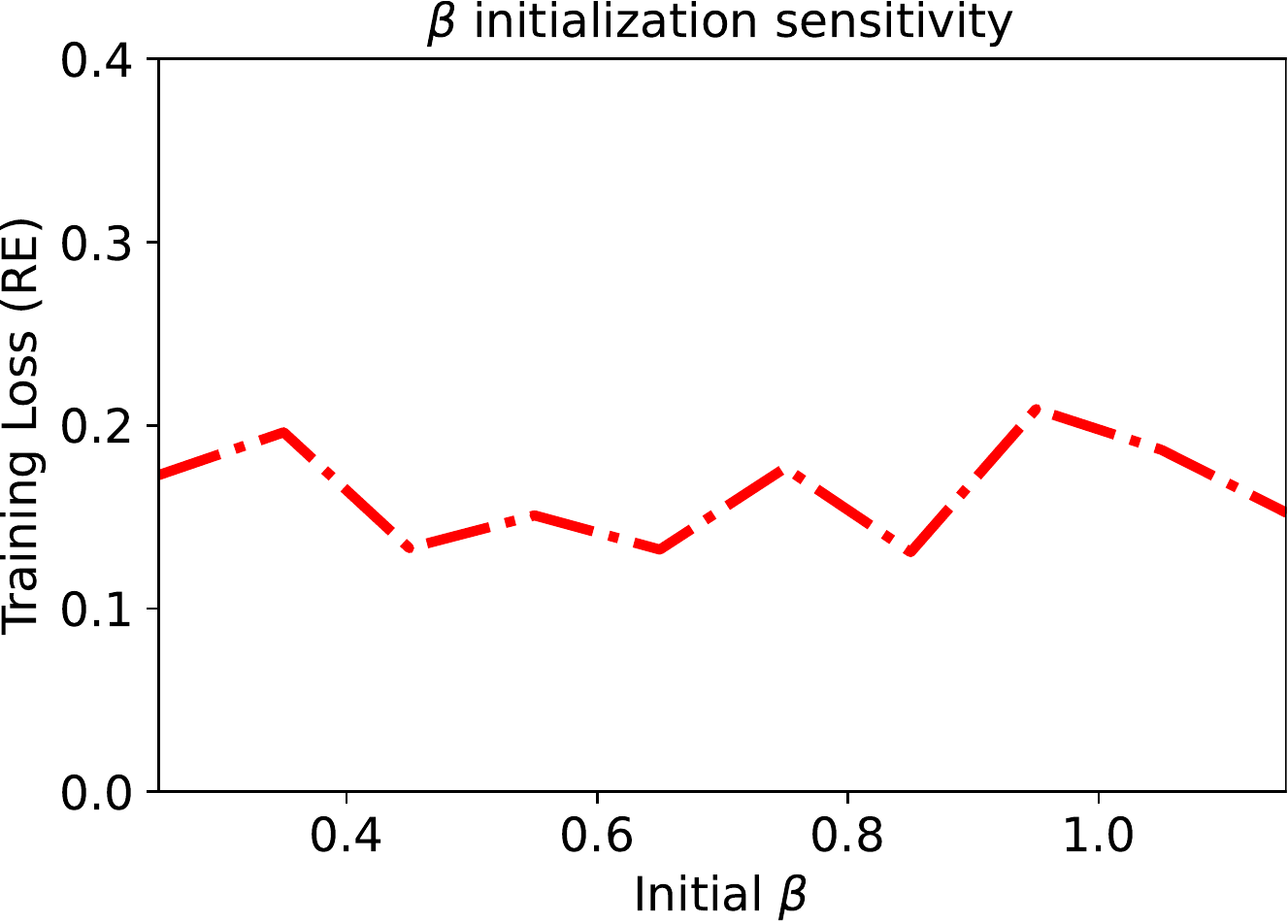}} \label{subfig: sensitivity smooth}  }\hspace{5pt}
\subfloat[Discontinuous Function]{%
\resizebox*{7.2cm}{!}{\includegraphics{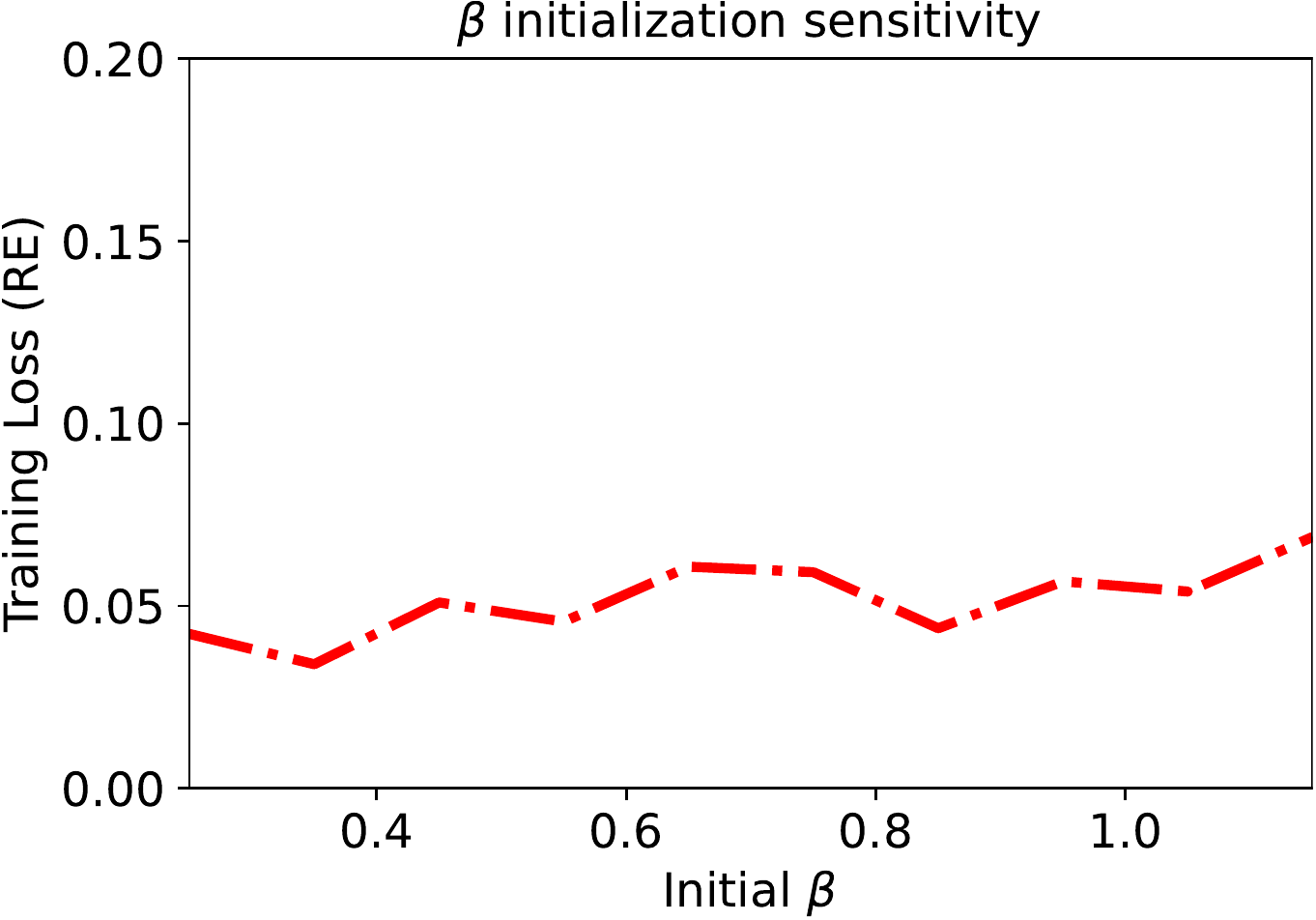}}  \label{subfig: sesnsitivity disc}}\\
\caption{Sensitivity of training under various initialization of $\beta^i_k$. Note that $\beta^i_k$s are initialized with the same value for all the neurons, $\beta$ is used to represent the value.} \label{fig: sensitivity analysis}
\end{figure} The training losses from the last epoch for both functions  (averaged over ten repetitions) in terms of RE are plotted in Figure~\ref{fig: sensitivity analysis} with different  initialization of $\beta_k^i$s. The results show that our proposed Stan is very stable training with different initialization of $\beta_k^i$s, which is particularly desired for more involved PINN case studies in Section~\ref{sec: PINN study}.

\section{Forward and Inverse Problems using PINN} \label{sec: PINN study}
In this Section, the proposed Stan activation is  tested on the multiple problems involving differential equations. These problems can be summarized into two types: (1) forward problem in Section~\ref{subsec: forward problem}, where the problem is to identify the solution of given a differential equation with its corresponding initial/boundary conditions (2) inverse problem in Section~\ref{subsec: inverse problem}, where the problem is to discover the coefficients of the differential equation given the solution. L-BFGS optimizer with the default learning rate 1 is used for all the cases in this section to evaluate the performance. All the results are repeated ten times with different initializations of NN weights and biases to obtain the average performance.

\subsection{Forward Problem} \label{subsec: forward problem}
\subsubsection{1-Dimensional Problems} \label{subsubsec: 1D problem}
As described in Section \ref{subsec: PINN}, given a specific form of the differential equation, the task of PINN is to solve the differential equation with the given initial/boundary conditions. The data provided to the NN is those on the boundary. A NN of nine hidden layers with 50 neurons in each layer is utilized.  The parameters $\beta^i_k$s are initialized with the value 1. Here, the first example considered is a second-order ODE as 
\begin{equation}\label{eq: 1D_1 differential equation}
    \begin{aligned}
              & \frac{d^2u}{dx^2} + \frac{du}{dx} = 6u,\,\, x  \in [0,2], \\
             & u(0) =2, \,\,  \frac{du}{dx}|_{x=0}= -1,
    \end{aligned}
\end{equation} 
where there is  a single point boundary (Dirichlet boundary) at the left end, namely, $x^1_u = 0$, and a first derivative boundary (Neumann boundary) at the same end. The differential equation residual $\mathcal{F}_{\bmTheta}$ needs to be zero for all points in the support. To enforce this condition, a set of 1,000 points ($x^i_f, i = 1,2,\dots, 1000$) are selected randomly for every training epoch.  Therefore, the loss is calculated by adding the loss of all the boundary conditions and the ODE residuals as in Eq.~\eqref{eq: 1D_1 Loss},
\begin{equation} \label{eq: 1D_1 Loss}
J(\tilde{\bmTheta}) =\frac{w_{\mathcal{F}}}{1000}\sum_{i=1}^{1000}\left| \mathcal{F}_{\tilde{\bmTheta}}(x_f^i) \right|^2 + w_u\left| u(0) - u_{\tilde{\bmTheta}}(0)\right|^2 + w_g \left| \mathcal{G}_{\tilde{\bmTheta}}(0)\right|^2,
\end{equation}
where $\mathcal{G}_{\tilde{\bmTheta}}(x^1_g) =  \frac{du_{\tilde{\bmTheta}}}{dx}|_{x=x^1_g} - \frac{du}{dx}|_{x=x^1_g}$ is an additional loss term corresponding to the Neumann boundary condition at $x^1_g = 0$ in Eq.~\eqref{eq: 1D_1 differential equation}.  $\mathcal{F}_{\tilde{\bmTheta}}(x_f^i)$ takes the following form
\begin{equation*} \label{eq: 1D_1 F and G}
      \mathcal{F}_{\tilde{\bmTheta}}(x^i_f) = \frac{d^2u_{\tilde{\bmTheta}}}{dx^2}|_{x=x^i_f} + \frac{du_{\tilde{\bmTheta}}}{dx}|_{x=x^i_f} -6u_{\tilde{\bmTheta}}(x^i_f), \,\, i = 1,2,\dots,1000.
\end{equation*}
For this case, the weights $w_f, w_u$, and $w_g$ are taken as $1$.

\begin{figure}[!htb]\vspace{-0.0cm}
\centering
\subfloat[Empirical convergence of training loss]{%
\resizebox*{6.8cm}{!}{\includegraphics{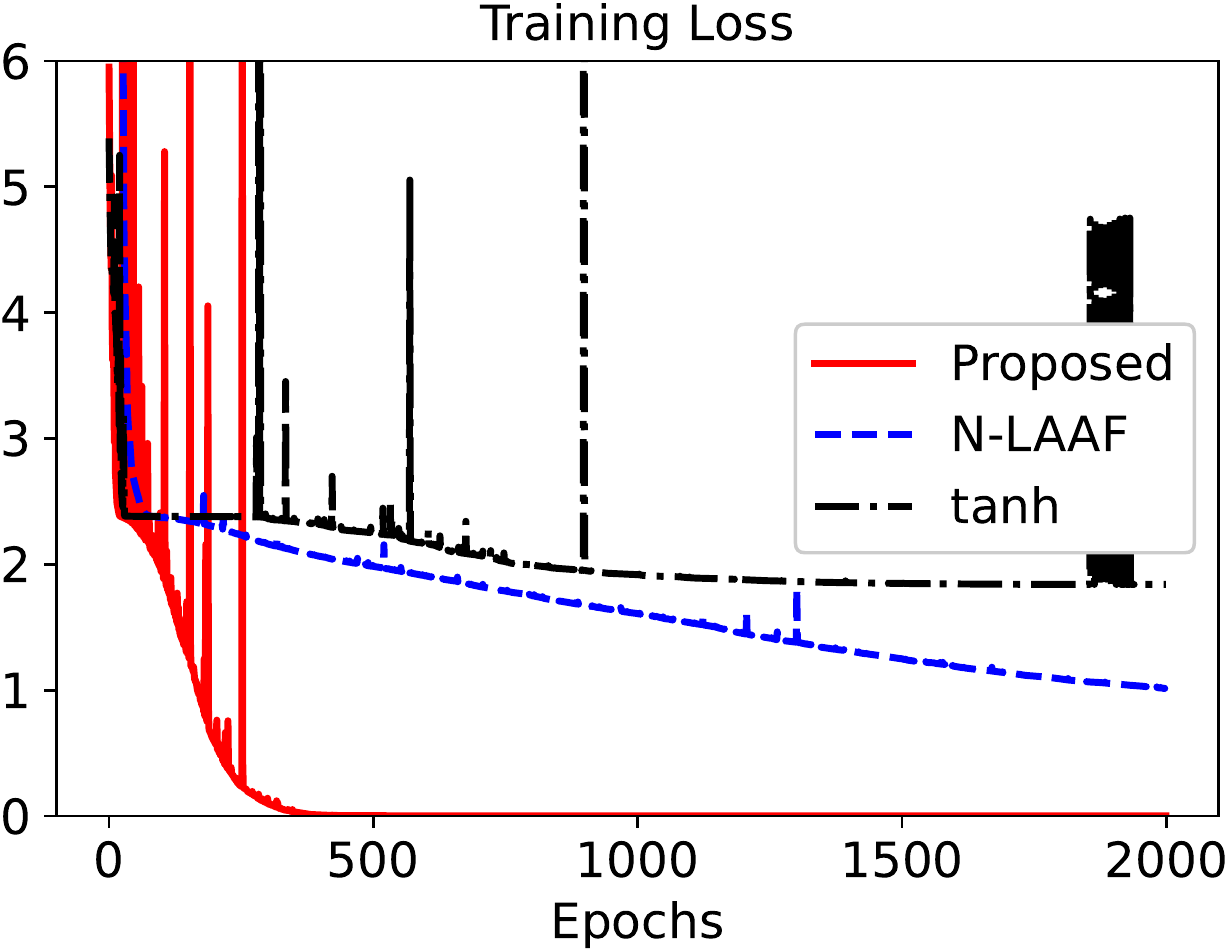}} \label{subfig: 1D-1-training loss} }\hspace{5pt}
\subfloat[Empirical convergence of $\beta^i_k$ in one of the neurons]{%
\resizebox*{7.8cm}{!}{\includegraphics{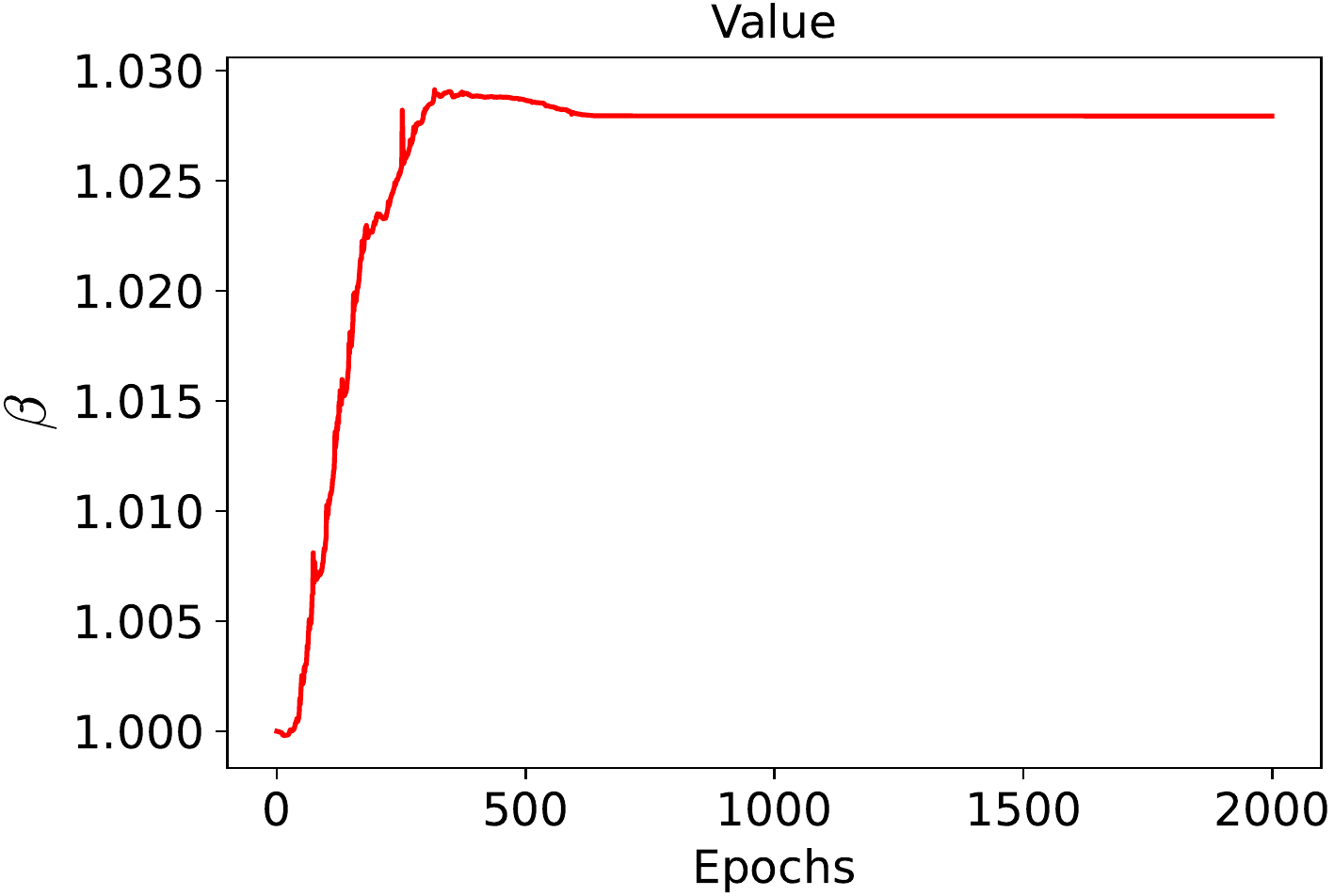}}  \label{subfig: 1D-1-beta} }\\
\subfloat[Neural network solution of the proposed method]{%
\resizebox*{5cm}{!}{\includegraphics{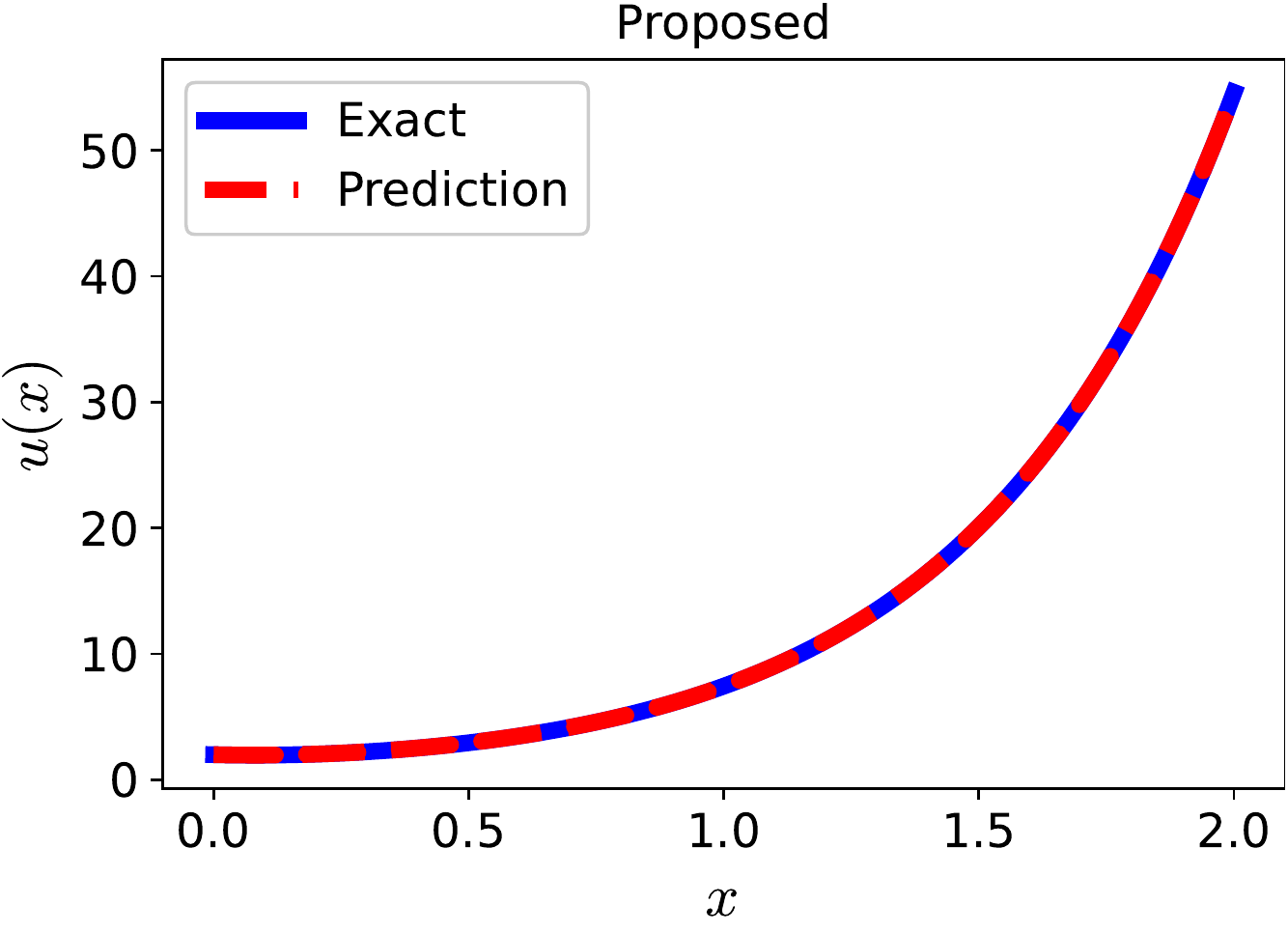}} \label{subfig: 1D-1-proposed} }
\subfloat[Neural network solution of N-LAAF]{%
\resizebox*{5cm}{!}{\includegraphics{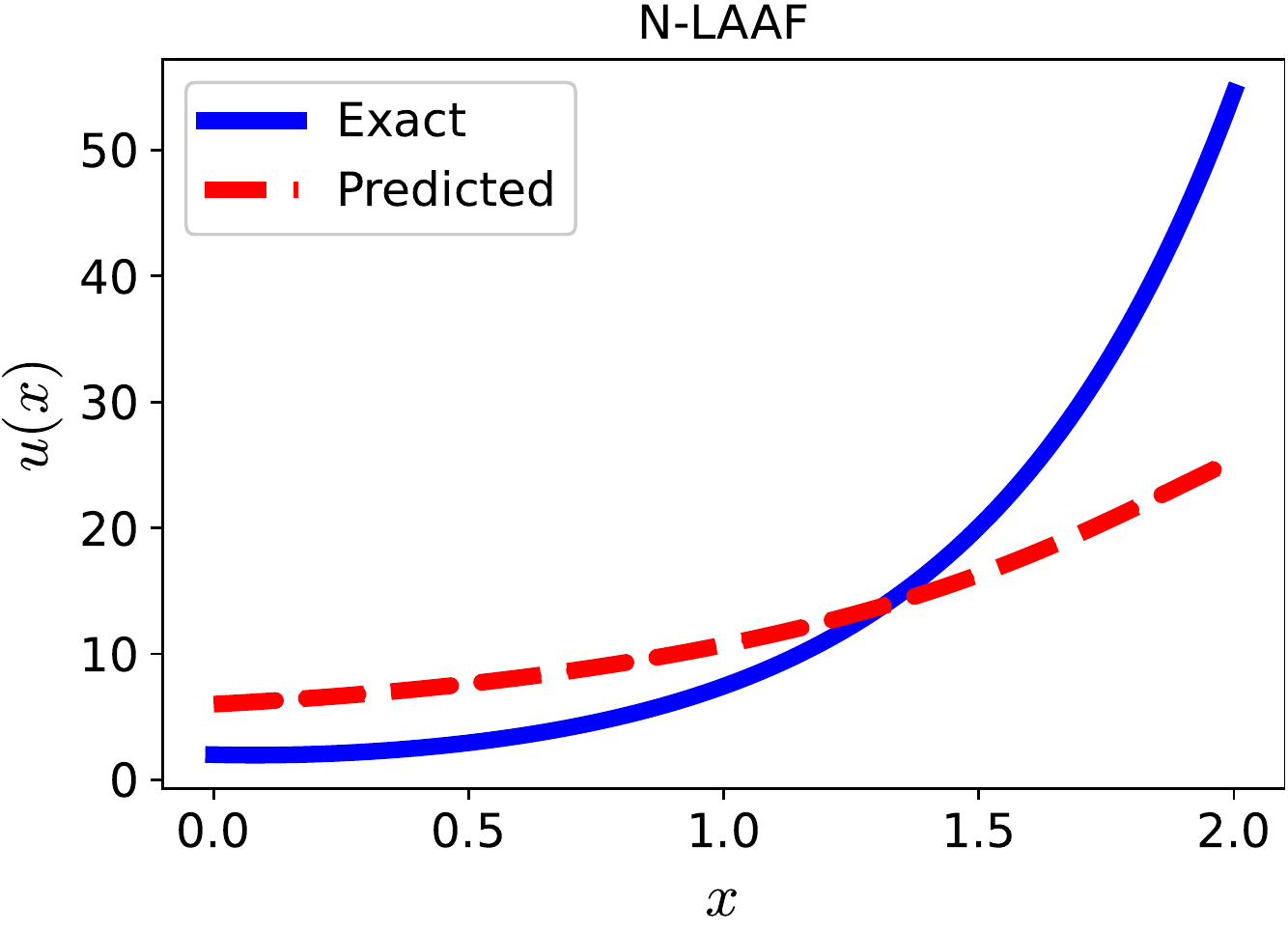}} \label{subfig: 1D-1-N-LAAF} }
\subfloat[Neural network solution of tanh]{%
\resizebox*{5cm}{!}{\includegraphics{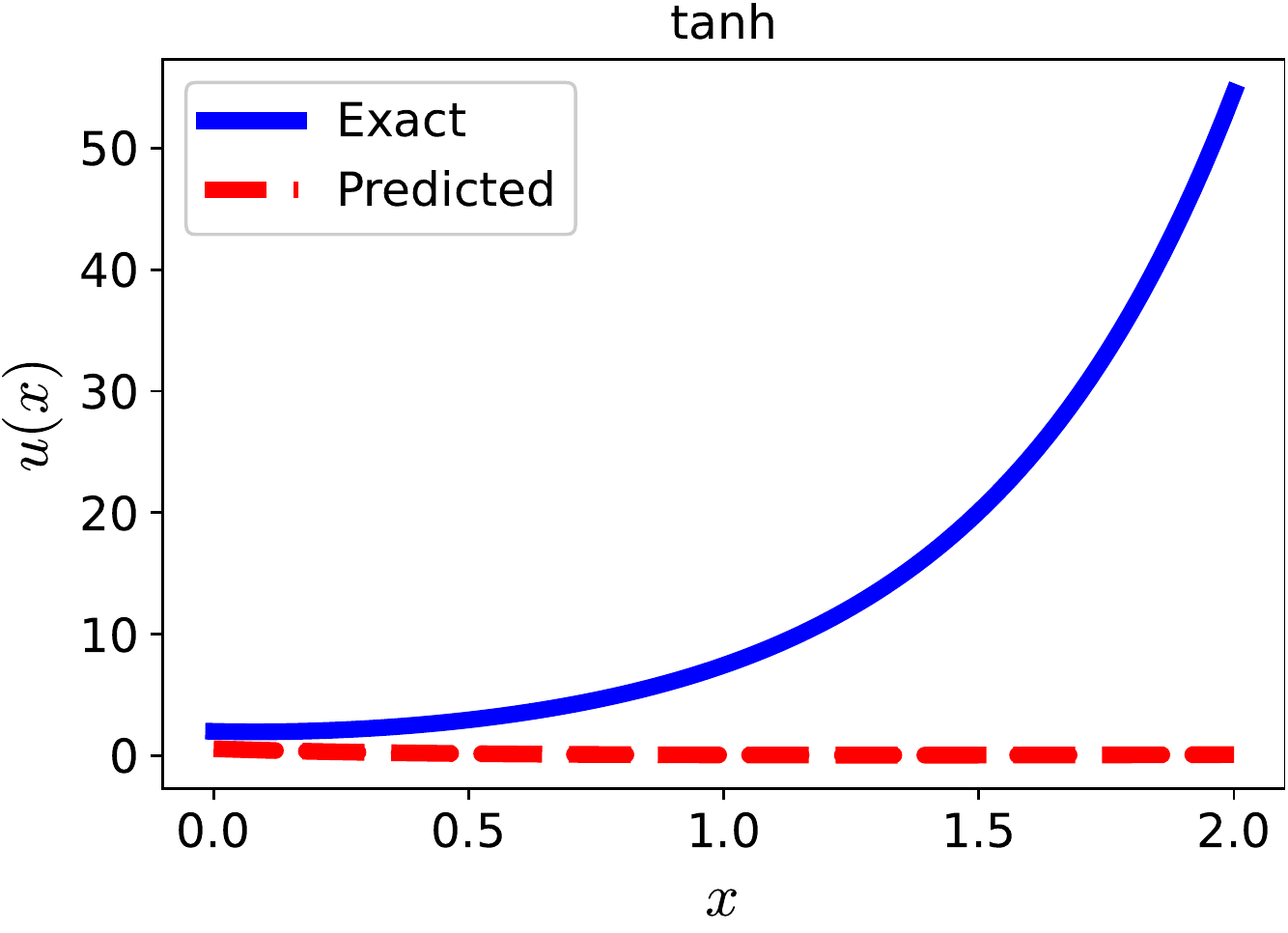}} \label{subfig: 1D-1-tanh} }
\caption{Results of Solving Eq.~\eqref{eq: 1D_1 differential equation}: (a) Empirical convergence of training loss; (b) Empirical convergence of $\beta^i_k$; (c) Neural network solution of the proposed method; (d) Neural network solution of N-LAAF; (e) Neural network solution of tanh.} \label{fig: performance of 1D-1}
\end{figure}
 As shown in Figure \ref{subfig: 1D-1-training loss}, the average loss in Eq.~\eqref{eq: 1D_1 Loss} converges to zero for the proposed activation function while other methods cannot.  The convergence of $\beta_k^i$ in one of the neurons is shown in Figure \ref{subfig: 1D-1-beta}.  It is well known that the problem in Eq.~\eqref{eq: 1D_1 differential equation} has a closed-form solution $u(x) = e^{2x} + e^{-3x}, \,\, x \in [0,2]$. The results of the neural network solution of Eq.~\eqref{eq: 1D_1 differential equation} are summarized in Figures \ref{subfig: 1D-1-proposed}, \ref{subfig: 1D-1-N-LAAF}, and \ref{subfig: 1D-1-tanh}. Our proposed method is the only one that can achieve the exact solution. Though this problem is simple, one-dimensional, and analytically solvable, the PINNs with tanh and N-LAAF activation functions fail to give an accurate NN model for this function.  This is mainly due to the difference in scale between $x$, which is $[0,2]$, and $u$, which is approximately $[0,55]$. Also, note that, with the given problem in Eq.~\eqref{eq: 1D_1 differential equation}, there is no way to identify the scale of $u$. The activation function N-LAAF, though improves the training, does not predict an accurate solution either.

\begin{figure}[!htb]
\centering
\subfloat[Empirical convergence of training loss]{%
\resizebox*{7cm}{!}{\includegraphics{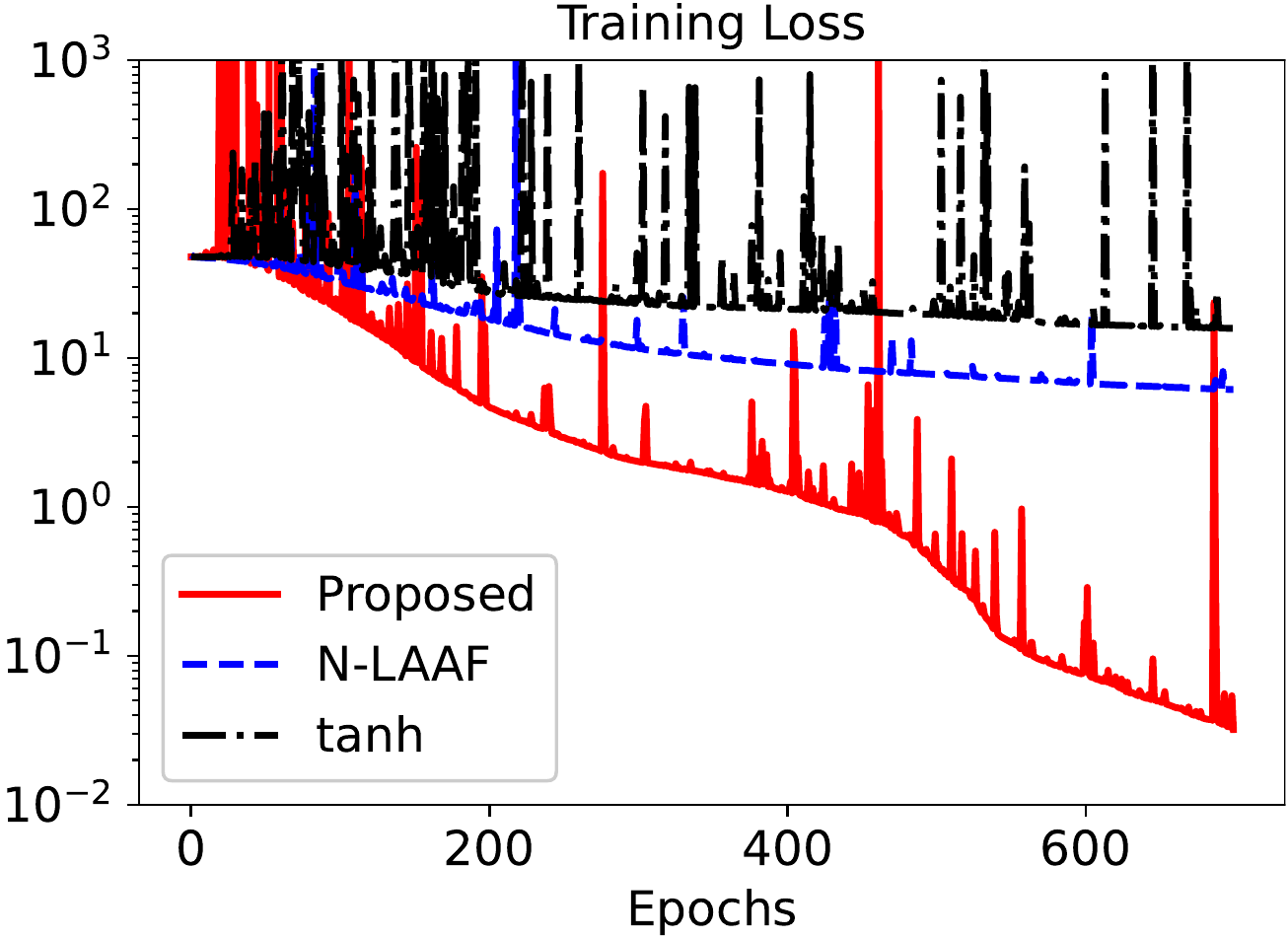}} \label{subfig: 1D-2-training loss}}\hspace{5pt}
\subfloat[Empirical convergence of $\beta^i_k$ in one of the neurons]{%
\resizebox*{7.4cm}{!}{\includegraphics{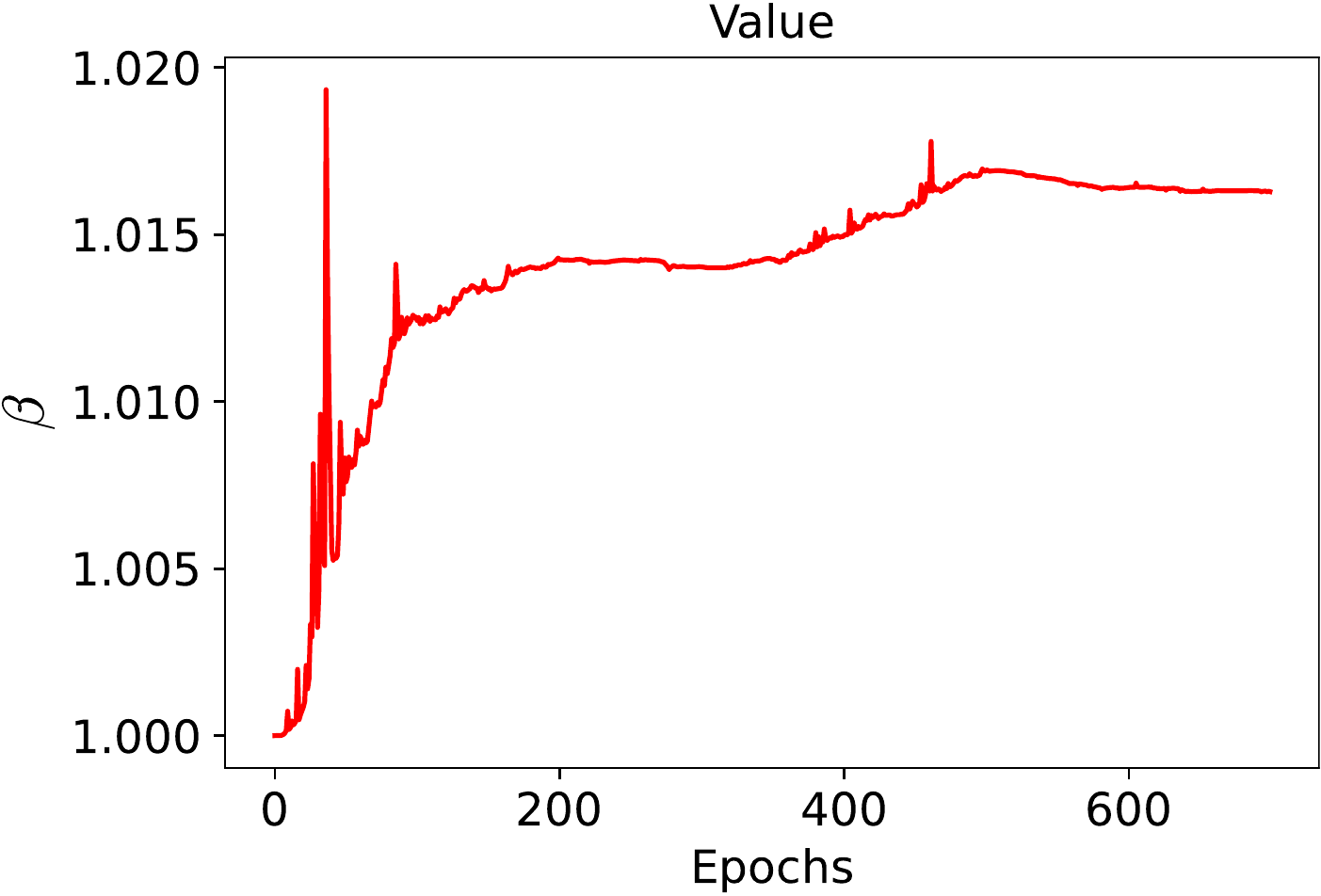}}  \label{subfig: 1D-2-beta} }\\
\subfloat[Neural network solution of the proposed method]{%
\resizebox*{5cm}{!}{\includegraphics{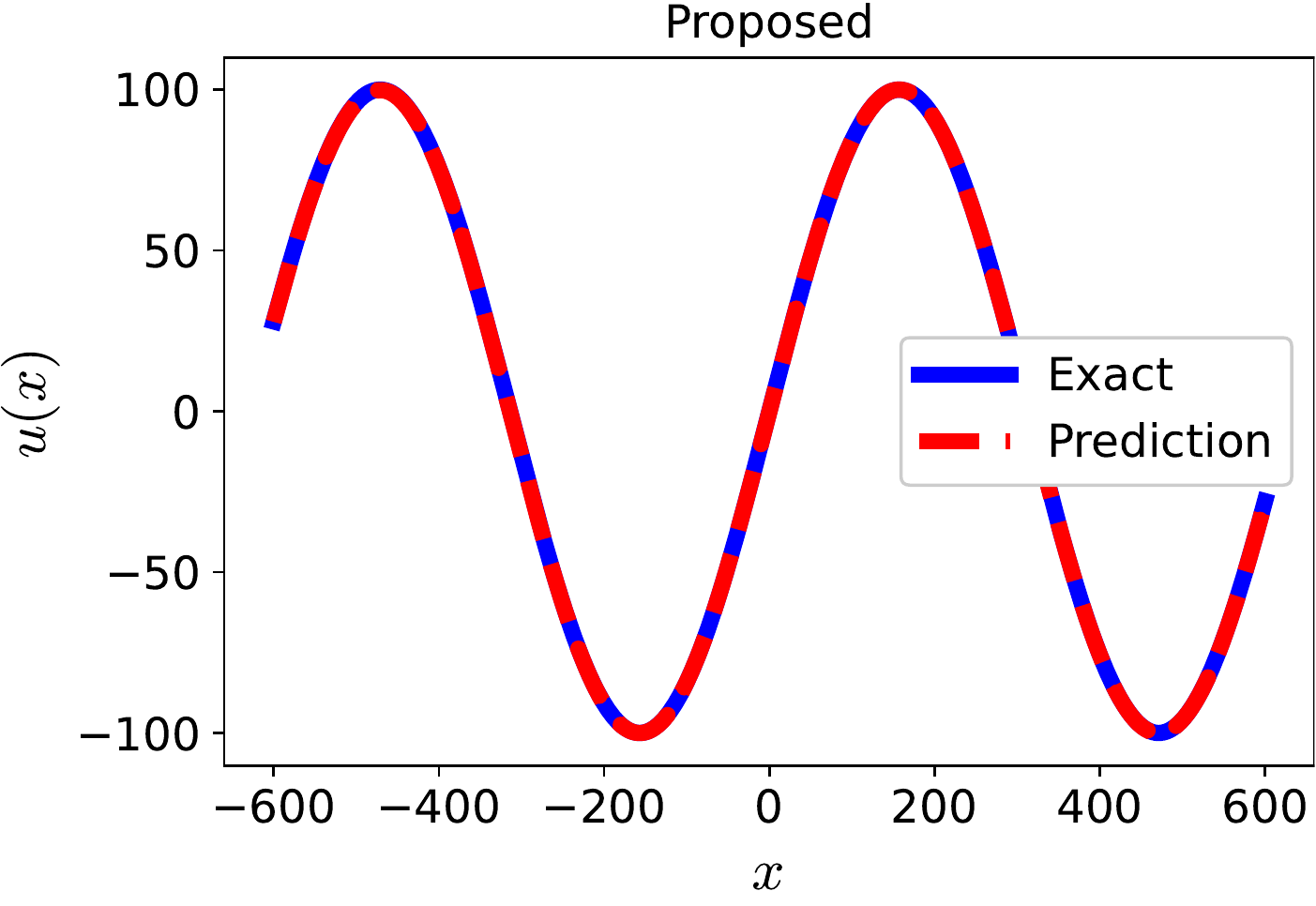}} \label{subfig: 1D-2-proposed} }
\subfloat[Neural network solution of N-LAAF]{%
\resizebox*{5cm}{!}{\includegraphics{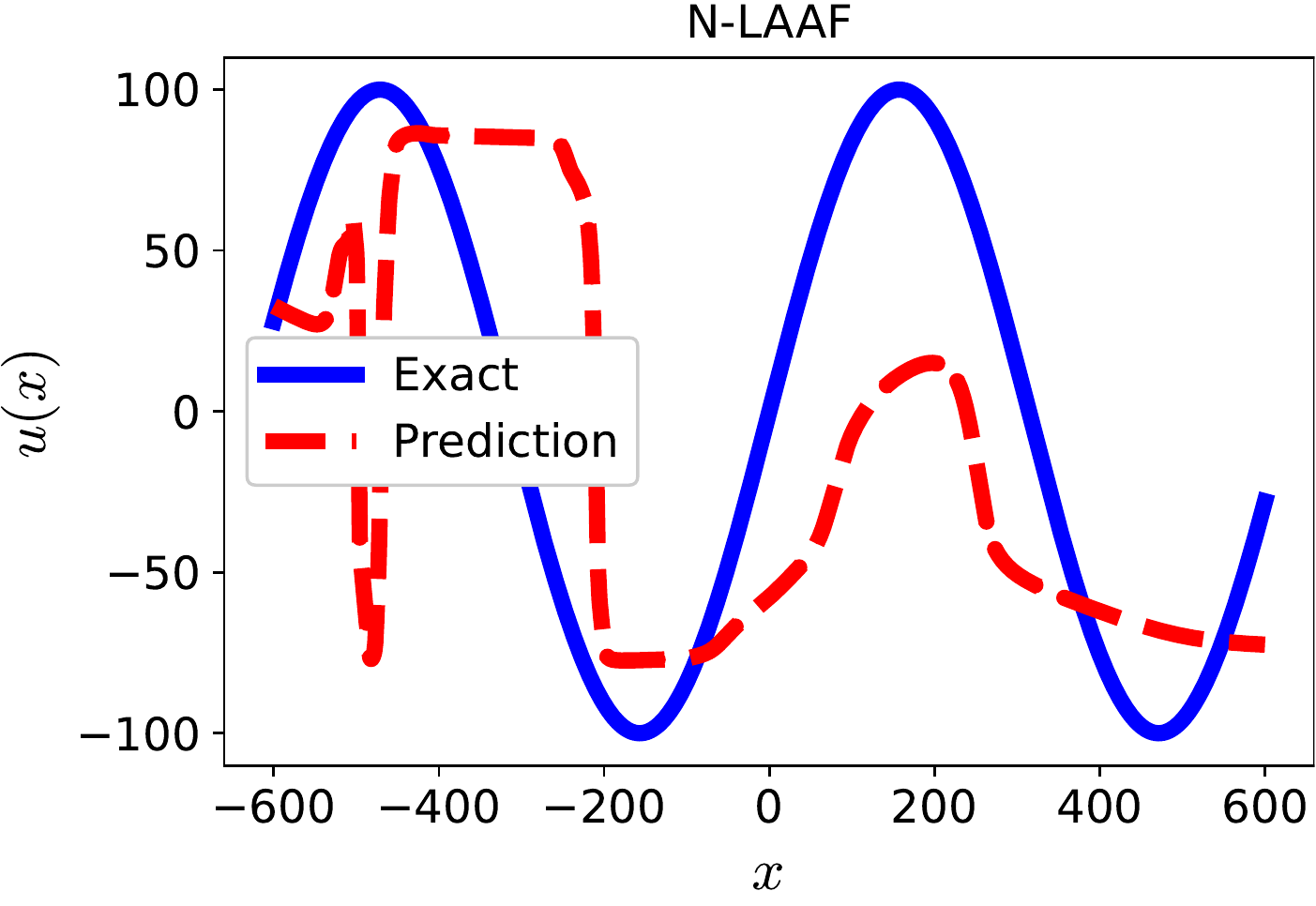}} \label{subfig: 1D-2-N-LAAF} }
\subfloat[Neural network solution of tanh]{%
\resizebox*{5cm}{!}{\includegraphics{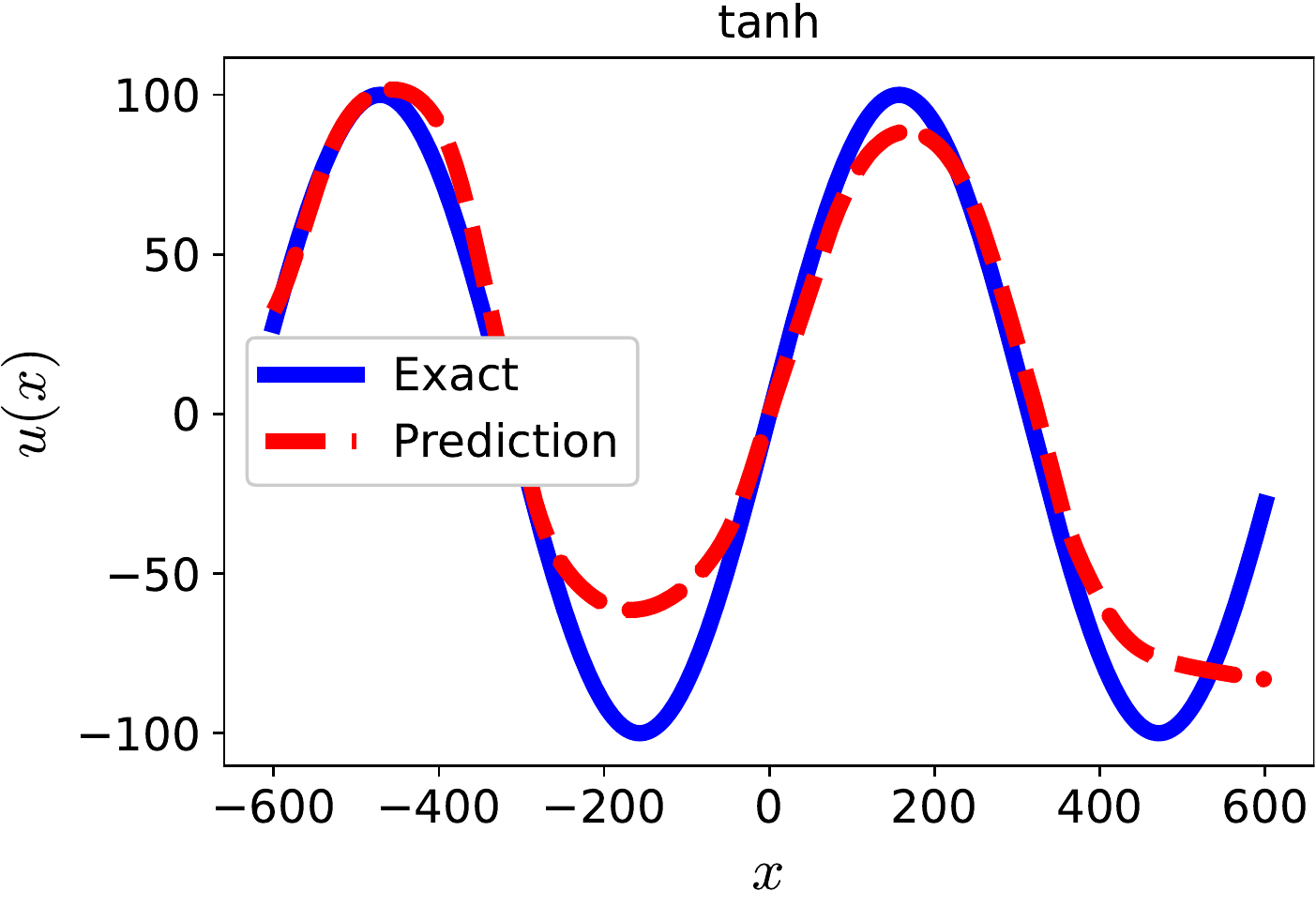}}  \label{subfig: 1D-2-tanh} }
\caption{Results of solving Eq.~\eqref{eq: 1D_2 differential equation}: (a) Empirical convergence of training loss; (b) Empirical convergence of $\beta^i_k$; (c) Neural network solution of the proposed method; (d) Neural network solution of N-LAAF; (e) Neural network solution of tanh.} \label{fig: performance of 1D-2}
\end{figure}

The second 1D example considered is inspired by the motivating example by \cite{moseley2021finite}. \cite{moseley2021finite} used a domain decomposition-based approach to address the spectral bias in the PINN methods. Unlike the first case, this is a first-order ODE with a single boundary condition as below
\begin{align}\label{eq: 1D_2 differential equation}
\begin{split}
 &\frac{du}{dx} = \cos(\omega x), \,\, x \in [-600,600], \\
 & u(0) = 0. 
\end{split}
\end{align}
Specifically, the problem is considered with very low frequency ($\omega = 0.01$). The loss function follows the formulation presented in Eq.~\eqref{eq: final loss function} as restated for this problem  as 
\begin{equation} \label{eq: 1D_2 Loss}
J(\tilde{\bmTheta}) =\frac{w_{\mathcal{F}}}{10000}\sum_{i=1}^{1000}\left| \mathcal{F}_{\tilde{\bmTheta}}(\bm{x}_f^i) \right|^2 + w_u\left| u(0) - u_{\tilde{\bmTheta}}(0)\right|^2, 
\end{equation}
where $\mathcal{F}_{\tilde{\bmTheta}}(x^i_f) =\frac{du_{\tilde{\bmTheta}}}{dx}|_{x^i_f} - \cos(0.01 x^i_f) $. The weights $w_{\mathcal{F}}$ and $w_u$ are chosen as 100 and 1, respectively, and 10,000 residual points are chosen randomly for every epoch. 

The problem in Eq.~\eqref{eq: 1D_2 differential equation} has the solution of the form $u(x) = 100\sin(0.01 x)$, where the magnitude of  $u$ is in the order of $10^2$.   Figure \ref{fig: performance of 1D-2} shows the algorithm convergence  and prediction performance of the proposed method in comparison with the benchmark methods. As shown in Figure \ref{subfig: 1D-2-training loss},  the proposed activation function is the only method that can achieve zero  average loss in Eq.~\eqref{eq: 1D_2 Loss}. In Figure~\ref{subfig: 1D-2-proposed}, the neural network solution of the proposed method can exactly match the solution. As seen in Figures \ref{subfig: 1D-2-N-LAAF} and \ref{subfig: 1D-2-tanh}, N-LAAF is the worst  in  predicting the solution even though its training loss (as shown in Figure \ref{subfig: 1D-2-training loss}) is better than the tanh activation. 

\begin{figure}[!htb]
\centering
\resizebox*{10cm}{!}{\includegraphics{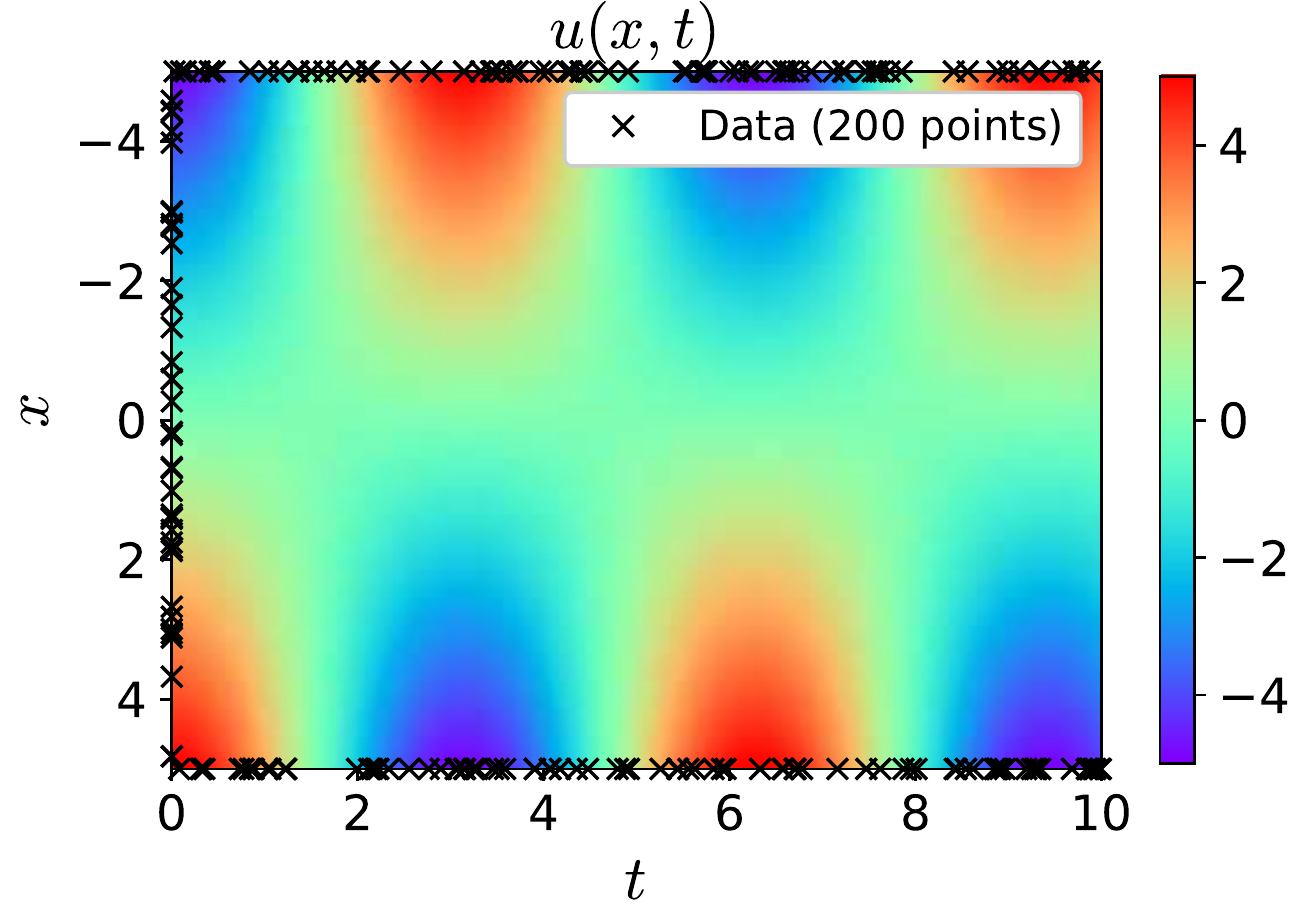}} \vspace{-0.5cm}
\caption{Ground truth solution of the Klein-Gordon equation. 200 ``$\times$''s are shown as training data for visualization purpose.} \label{fig: ground truth solution Klein-Gordon}
\end{figure}
\begin{figure}[!htbp]
\centering
\subfloat[Empirical convergence of training loss]{%
\resizebox*{7cm}{!}{\includegraphics{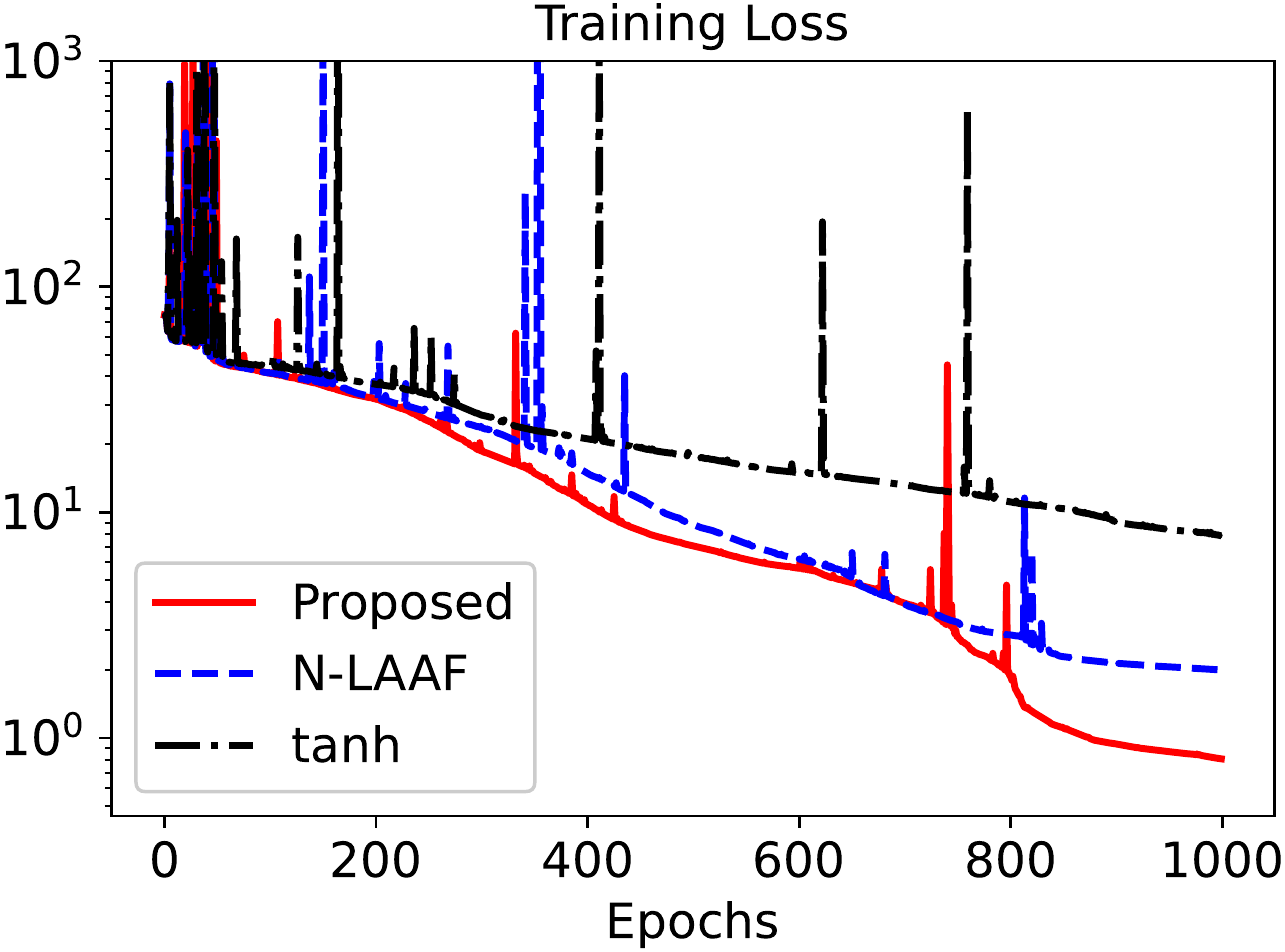}} \label{subfig: Klein-Gordon-training loss} }\hspace{5pt}
\subfloat[Empirical convergence of $\beta^i_k$ in one of the neurons]{%
\resizebox*{7.5cm}{!}{\includegraphics{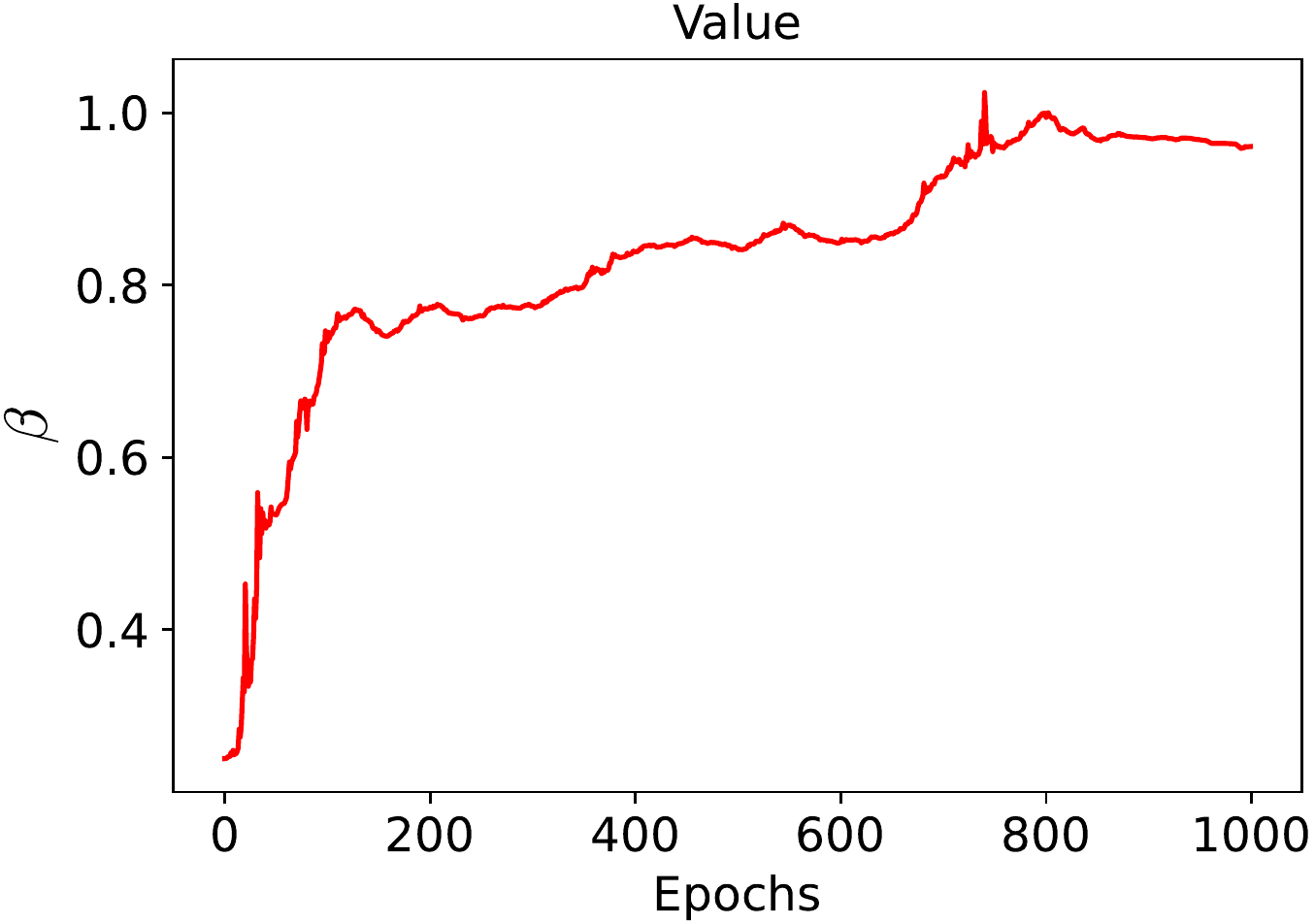}}  \label{subfig: Klein-Gordon-beta} }\\
\subfloat[Neural network solutions of the proposed method, N-LAAF, and tanh]{%
\resizebox*{15cm}{!}{\includegraphics{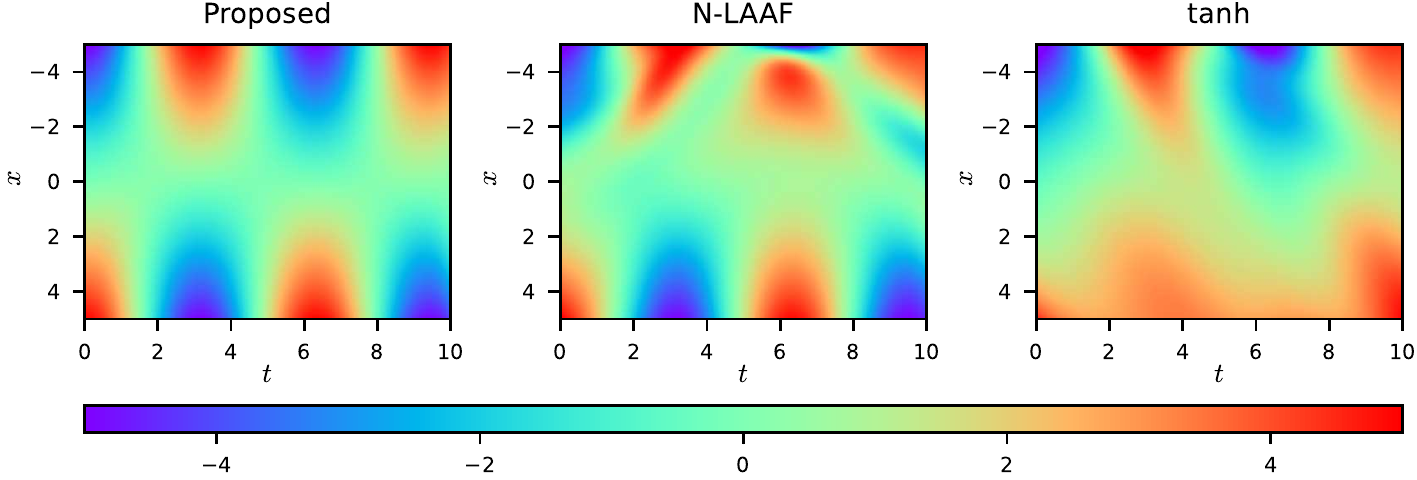}} \label{subfig: Klein-Gordon-solution} }\\
\subfloat[Absolute error maps for neural network solutions of the proposed method, N-LAAF, and tanh]{%
\resizebox*{15cm}{!}{\includegraphics{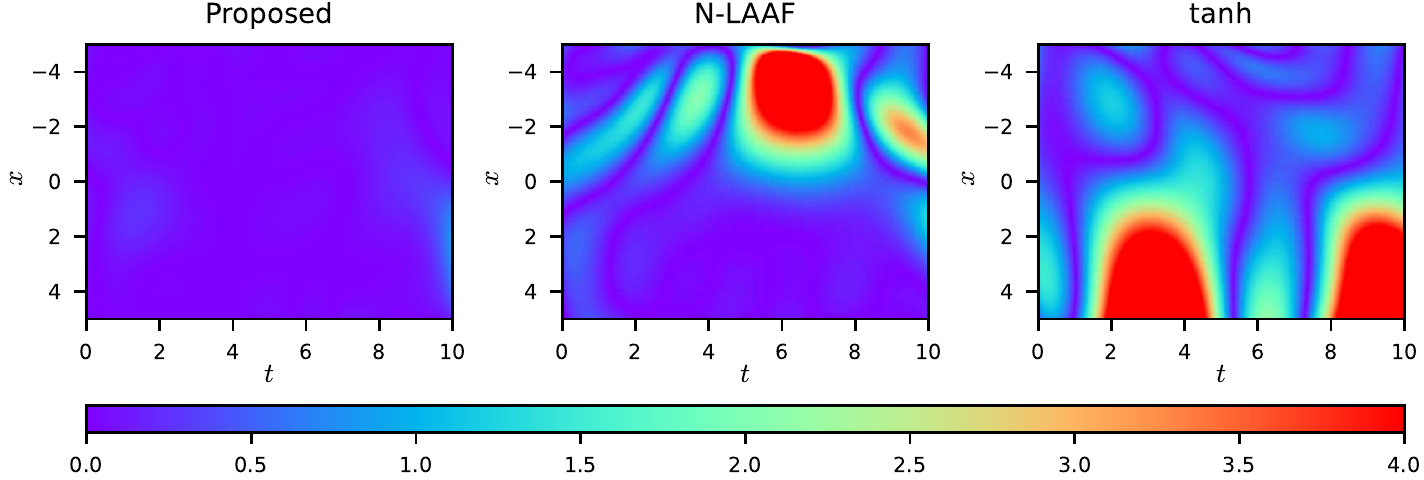}} \label{subfig: Klein-Gordon-errorMap} }  \caption{Results of  the Klein-Gordon equation: (a) Empirical convergence of training loss; (b) Empirical convergence of $\beta^i_k$; (c) Neural network solutions of the proposed method, N-LAAF, and tanh; (d) Neural network solution of N-LAAF; (e) Absolute error maps for neural network solutions of the proposed method, N-LAAF, and tanh.} \label{fig: performance of Klein-Gordon}
\end{figure}
\subsubsection{2-Dimensional Problem} \label{subsubsec: 2D problem}
As described in Section \ref{sec: proposed activation}, the Klein-Gordon equation is a second-order partial differential equation \citep{caudrey1975sine} arising in many branches of physics. The inhomogeneous Klein-Gordon equation, as given in Eq.~\eqref{eq: Klein-Gordon}, is used in this case study with the conditions given in~\cite{jagtap2020adaptive} but for an extended domain (i.e., $x \in  [-5,5]$) and the same time interval (i.e., $t \in [0,10]$). Specifically, the exact problem considered is given by
\begin{align} \label{eq: 2D_1 differential equation}
\begin{split}
&\frac{\partial^2 u}{\partial t^2} -  \frac{\partial^2 u}{\partial x^2} + u^2 = -x\cos(t) + x^2\cos^2(t), \,\, x \in [-5,5], t>0, \\
&u(x,0) = x;\,\, u(-5,t) = -5\cos(t);\,\, u(5,t) = 5\cos(t). \\
\end{split}
\end{align}
The above problem in Eq.~\eqref{eq: 2D_1 differential equation} has a closed-form solution  given by $u(x) = x\cos(t), \,\, x \in [-5,5]$ as shown visually in Figure~\ref{fig: ground truth solution Klein-Gordon}.  For the purpose of training, 500 points are chosen randomly on the boundaries (comprising of all the points where $x=-5$ or $x=5$ or $t=0$) for the \texttt{Data loss} in Eq.~\eqref{eq: definition MSE_u}. 10,000 residual points are randomly chosen from the support domain for the \texttt{PDE loss} in Eq.~\eqref{eq: definition MSE_f}. The neural network architecture used for this case study consists of nine hidden layers with 50 neurons in each layer.  All $\beta_k^i$s  are initialized with 0.25 and weights $w_f, w_u$ are taken as 1. 


Figure~\ref{fig: performance of Klein-Gordon} summarizes the convergence and prediction performance of different methods. Figure~\ref{subfig: Klein-Gordon-training loss} shows that our approach has a similar convergence speed with N-LAAF for the first 800 epochs but converges to a better solution in the end. Figures \ref{subfig: Klein-Gordon-solution} and \ref{subfig: Klein-Gordon-errorMap} correspond to the neural network solutions and absolute error maps for different methods. It can be observed visually that the proposed method outperforms the benchmark methods.  On the other hand, N-LAAF has better convergence than the tanh activation. This case study illustrates that even without changing any coefficients/parameters of Klein-Gordon equation from~\cite{jagtap2020adaptive}, simply by extending the domain (from $ x\in [-1,1]$ to $x\in[-5,5])$, the output can be of higher magnitudes and the benchmark methods fail to scale up.

\subsection{Inverse Problem} \label{subsec: inverse problem}
Another significant problem related to PINN is the inverse problem, which is to find the differential equation itself with the given data \cite{raissi2019physics}. Specifically, it aims to find the coefficients/parameters of the differential equation, which  is also called the data-driven discovery of partial differential equations. This problem is encountered practically in many cases. For instance, sensor data can be used to calculate the material property or characteristic of a given work-piece~\citep{harrsion2020material, crawford2018advanced}. 

To demonstrate the effectiveness of the proposed activation function in the inverse problem,   the transient heat transfer in a rod is considered~\citep{cengel2008introduction}. The rod is considered one-dimensional with a length of one unit. It has a constant initial temperature of $u_0$ and starts to cool down at time $t=0$. The ends of the rod are assumed to be maintained at a temperature of zero units. Thermal effects except the conduction are ignored. Specifically, the conduction heat transfer equation~\citep{cengel2008introduction} given by the PDE with initial  and boundary conditions, where $u(x,t)$ is the temperature that follows
\begin{align} \label{eq: heat transfer equation}
    \begin{split}
        \texttt{PDE}&:\,\, \frac{\partial u}{\partial t} = \kappa \frac{\partial^2 u}{\partial x^2},\,\, 0<x<1, \\
        \texttt{Initial Condition}&:\,\, u(x,0) = u_0, \,\, 0<x<1,\\
        \texttt{Boundary Condition}&:\,\, u(0,t) = u(1,t) = 0, \,\, t>0,
    \end{split}
\end{align}
where the thermal diffusivity $\kappa$ is  a constant for a given rod. 
In this case, it is assumed that the $\kappa$ is unknown. The goal is \textbf{\textit{to find the thermal diffusivity of the rod (i.e., $\kappa$) based on the thermal measurement at specific points during the transition to equilibrium temperature}}. From the physics perspective and/or the PINN perspective, the objective is to identify the coefficient of the differential equation denoted as $\kappa$ (i.e., the thermal diffusivity)  given the approximate temperature data (i.e., $\bar{u}(x,t)$) spread over the entire domain. In addition, PINN can produce a better approximation (i.e., $u_{\bmTheta}(x,t)$)  of the temperature distribution (i.e., solution $u(x,t)$) for the rod as a function of spatial location $x$ and time $t$.

\begin{figure}[!htbp]
\centering
\resizebox*{10cm}{!}{\includegraphics{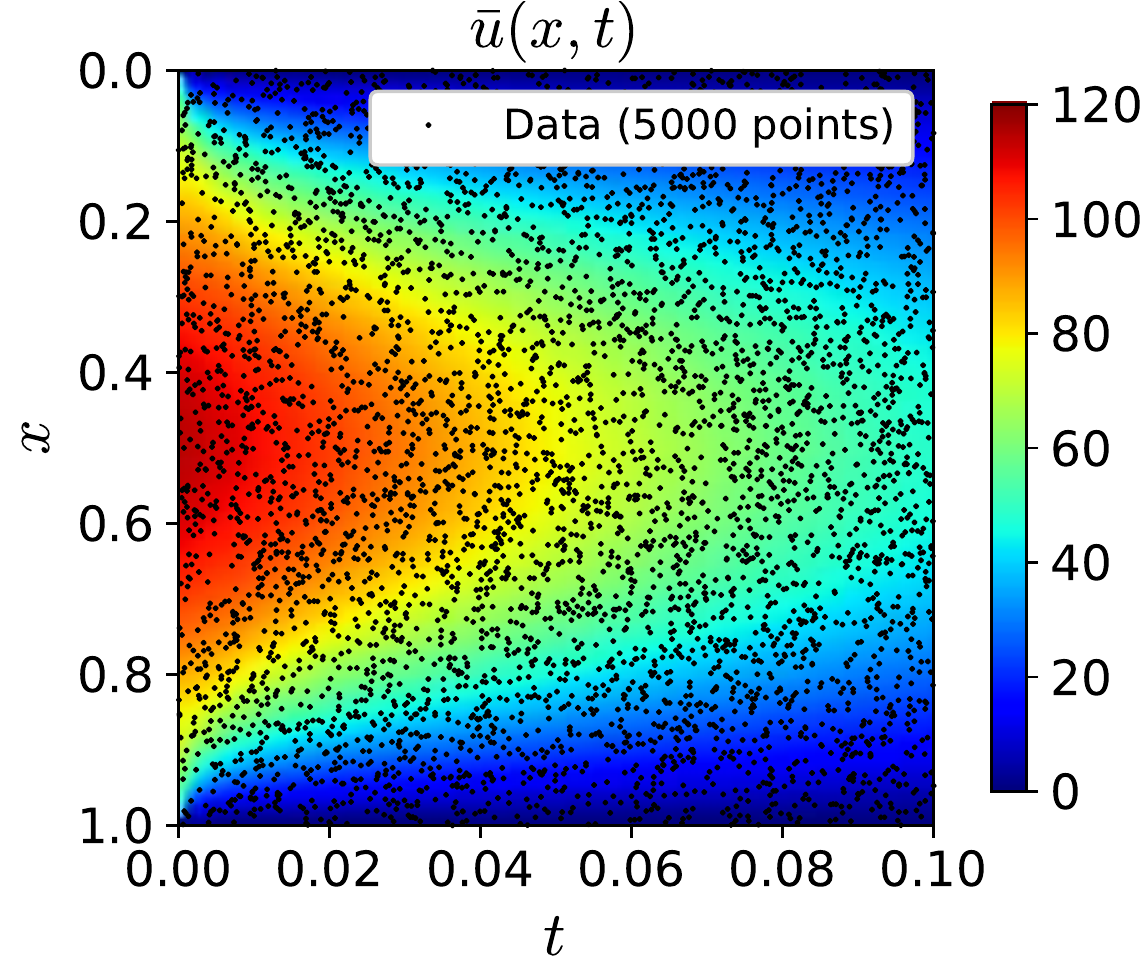}} \vspace{-0.5cm}
\caption{Approximate solution of the transient 1D Heat Transfer equation calculated by closed-form expression. The ``$\cdot$''s represent training data.} \label{fig: ground truth solution Transient 1D heat Transfer}
\end{figure}
To simulate this case,  the thermal diffusivity  $\kappa^* = 1$ is considered the \textbf{ground truth} and the initial temperature $u_0$ of the rod is set as 50 units. By choosing $\kappa^*=1$, it leads to sufficient cooling within $t = 0.1$ units. Therefore, the data within that time interval is sufficient to consider cooling. For this particular case, the analytical solution~\citep{cengel2008introduction} is given by an infinite series as \begin{equation}\label{eq: heat transfer solution}
    u(x,t) = \frac{4 u_0}{\pi} \sum_{i=1}^{\infty} \frac{\sin\Big((2i-1)\pi x\Big)}{2i-1} \exp{\Big(-(2i-1)^2 \pi^2 t\Big)}.
\end{equation} 
Choosing any finite number of terms from the infinite series results in an approximation $\bar{u}(x,t)$ of the actual function $u(x,t)$. In this work, the first 10,000 terms of the series were used to calculate the approximation $\bar{u}(x,t)$, which is accurate enough~\citep{cengel2008introduction}.  5,000 randomly chosen spatial-temporal locations from $0 \leq x \leq 1,0 \leq t \leq 0.1$, and their corresonding $\bar{u}(x,t)$ are treated as thermal measurement data for training. The visualization of the approximate solution and thermal measurement data is shown in Figure \ref{fig: ground truth solution Transient 1D heat Transfer}. 

To find the thermal diffusivity $\kappa$, the loss function is based on the one in Eq. \eqref{eq: final loss function} with an additional  trainable  parameter $\kappa$ in the residual part, which has the following form.
    \begin{equation}\label{Inverse Loss function}
        J(\tilde{\bmTheta}, \kappa) =\frac{w_{\mathcal{F}}}{50000}\sum_{i=1}^{50000}\left| \mathcal{F}_{\tilde{\bmTheta},\kappa}(\bm{x}_f^i) \right|^2 + \frac{w_u}{5000}\sum_{i=1}^{5000}\left| \bar{u}^i(\bm{x}_u^i) - u_{\tilde{\bmTheta}}(\bm{x}_u^i)\right|^2,
    \end{equation}
where $\mathcal{F}_{\tilde{\bmTheta},\kappa}(\bm{x}_f^i) = \frac{\partial u_{\tilde{\bmTheta}}}{\partial t}|_{\bm{x}_f^i} - \kappa \frac{\partial^2 u_{\tilde{\bmTheta}}}{\partial x^2}|_{\bm{x}_f^i}$,  $w_u=w_\mathcal{F}=1$, and a set of 50,000 points are used as residual points to ensure the PDE is satisfied. This leads to training $\kappa$ along with the neural network parameters (including those of activation functions). The neural network architecture used for this numerical study consists of ten hidden layers with 50 neurons in each layer.   The $\beta_k^i$s are initialized as 1 for the proposed activation function.  The $\kappa$ value is initialized as 0.

\begin{figure}[!htb]
\centering
\resizebox*{8cm}{!}{\includegraphics{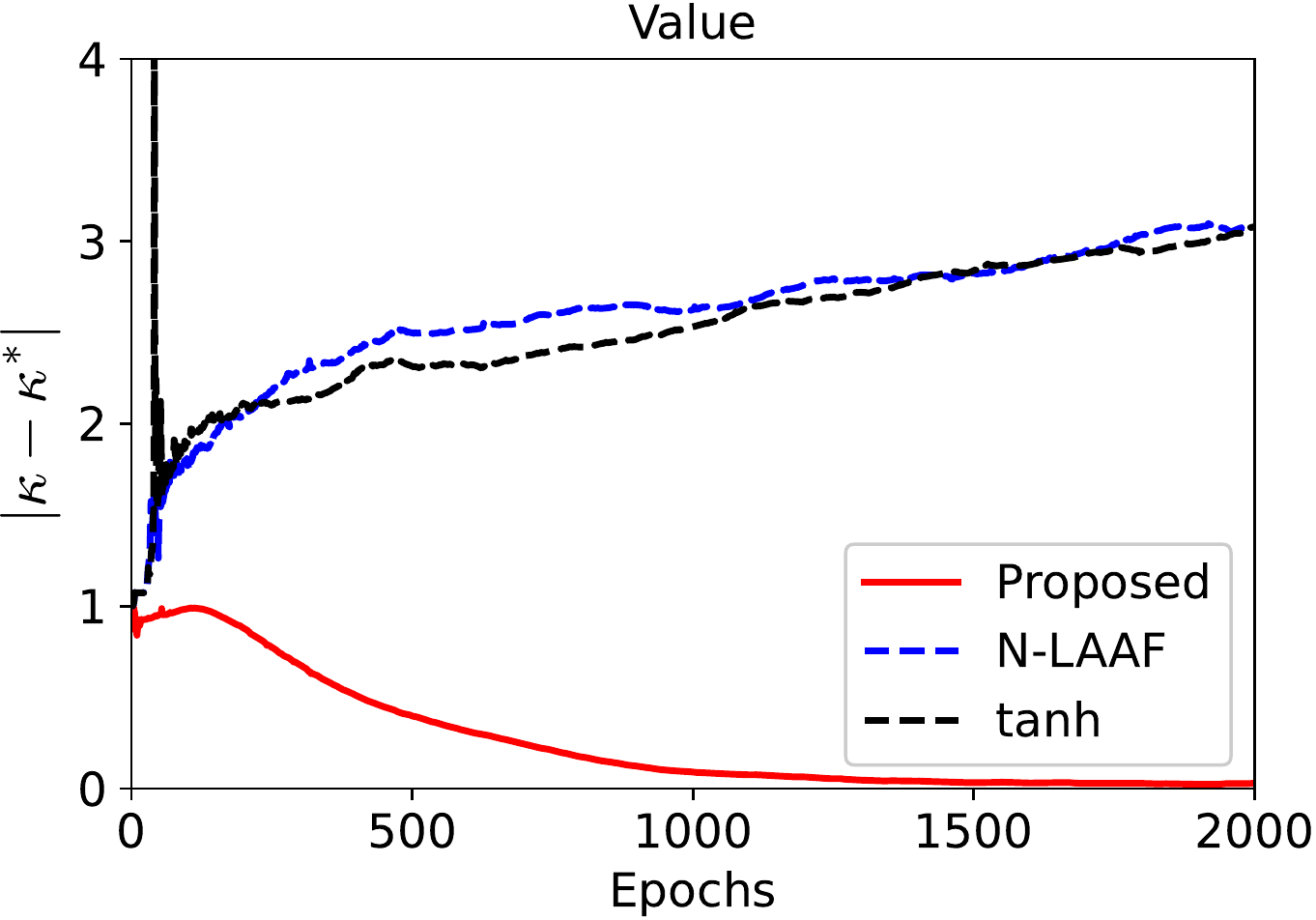}} \vspace{-0.5cm}
\caption{Convergence of $\left|\kappa-\kappa^*\right|$ of the Transient 1D Heat Transfer equation.} \label{fig: performance of kappa in Transient 1D heat Transfer}
\end{figure}
\begin{figure}[!htbp]\vspace{-1.5cm}
\centering
\subfloat[Empirical convergence of training loss]{%
\resizebox*{6.8cm}{!}{\includegraphics{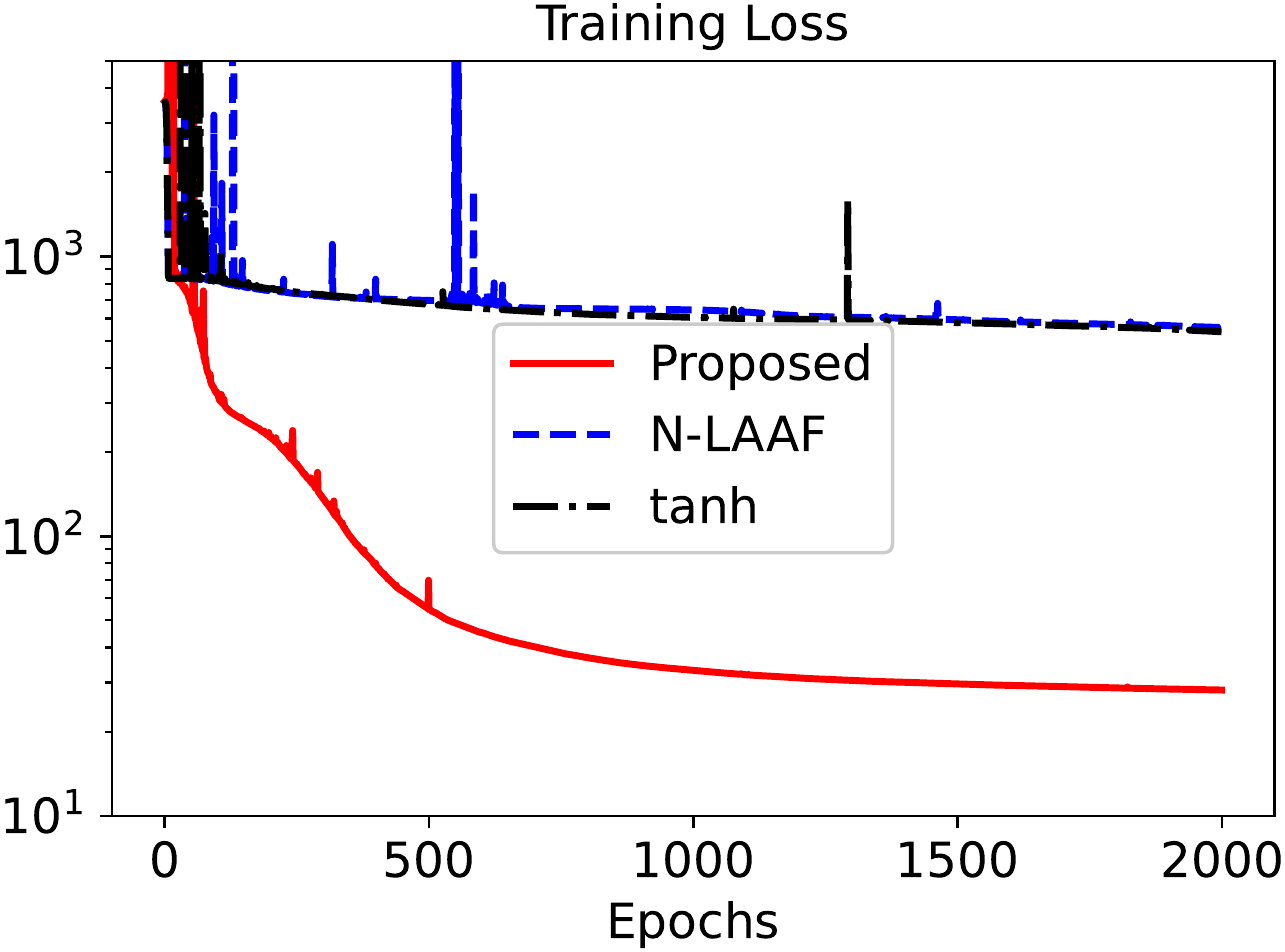}} \label{subfig: Transient 1D heat Transfer-training loss} }\hspace{5pt}
\subfloat[Empirical convergence of $\beta^i_k$ in one of the neurons]{%
\resizebox*{7.5cm}{!}{\includegraphics{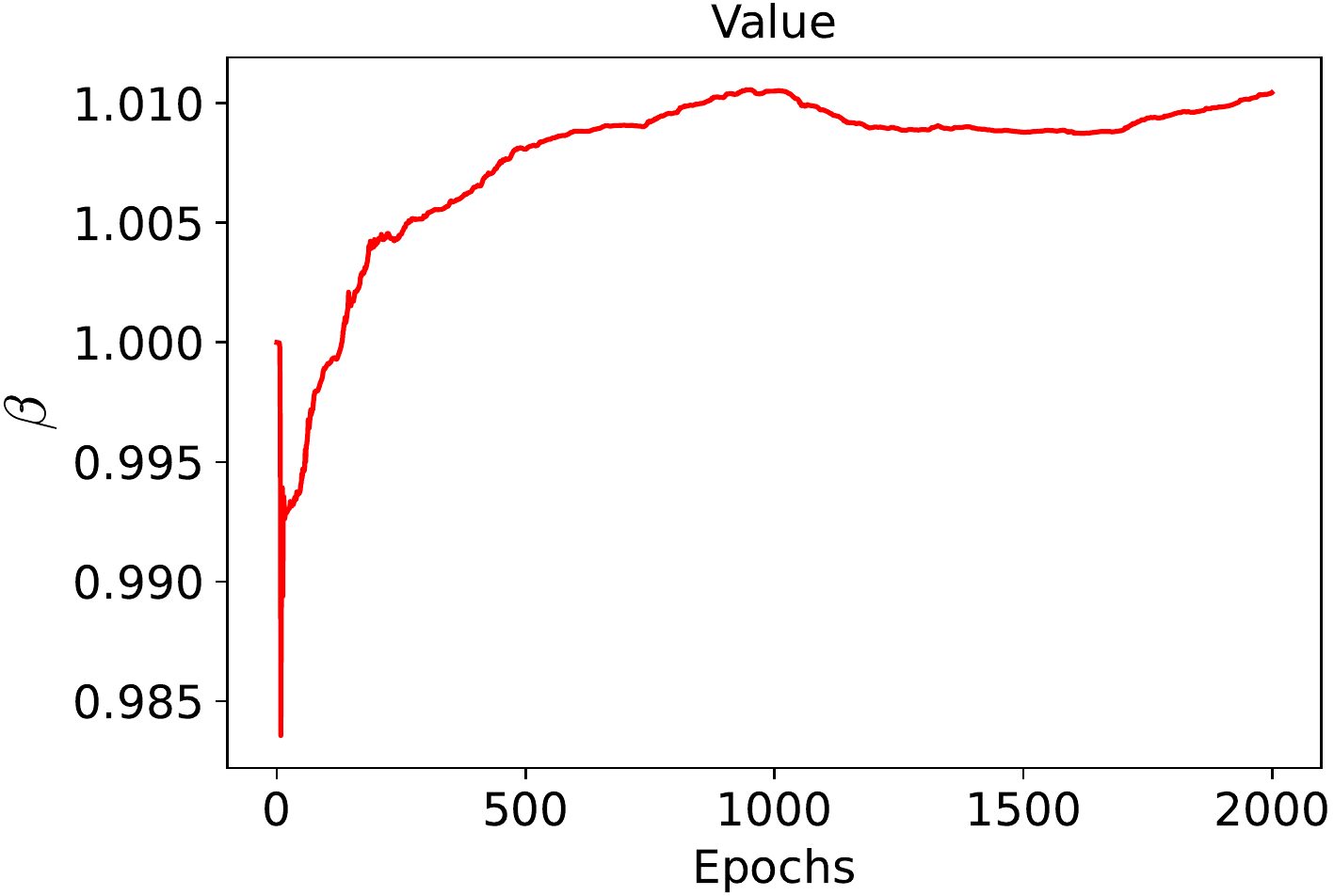}}  \label{subfig: Transient 1D heat Transfer-beta} }\\
\subfloat[Neural network solutions of the proposed method, N-LAAF, and tanh]{%
\resizebox*{15cm}{!}{\includegraphics{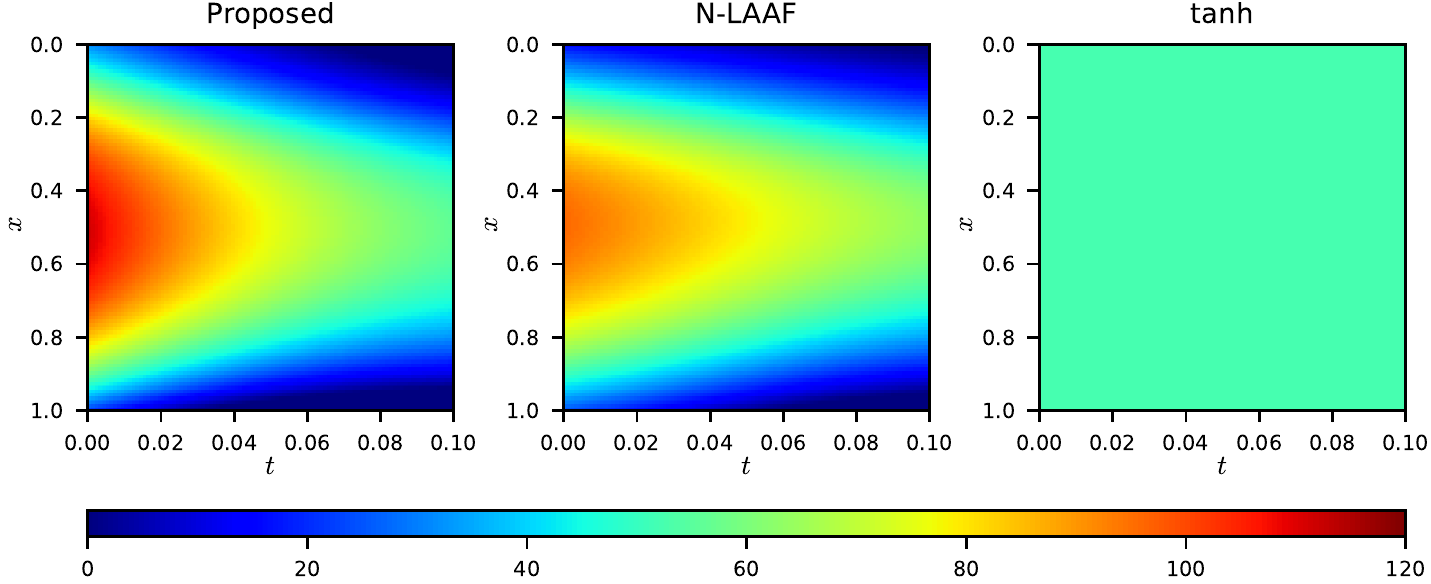}} \label{subfig: Transient 1D heat Transfer-solution} }\\
\subfloat[Absolute difference maps for neural network solutions of the proposed method, N-LAAF, and tanh when compared to the approximate closed form solution]{%
\resizebox*{15cm}{!}{\includegraphics{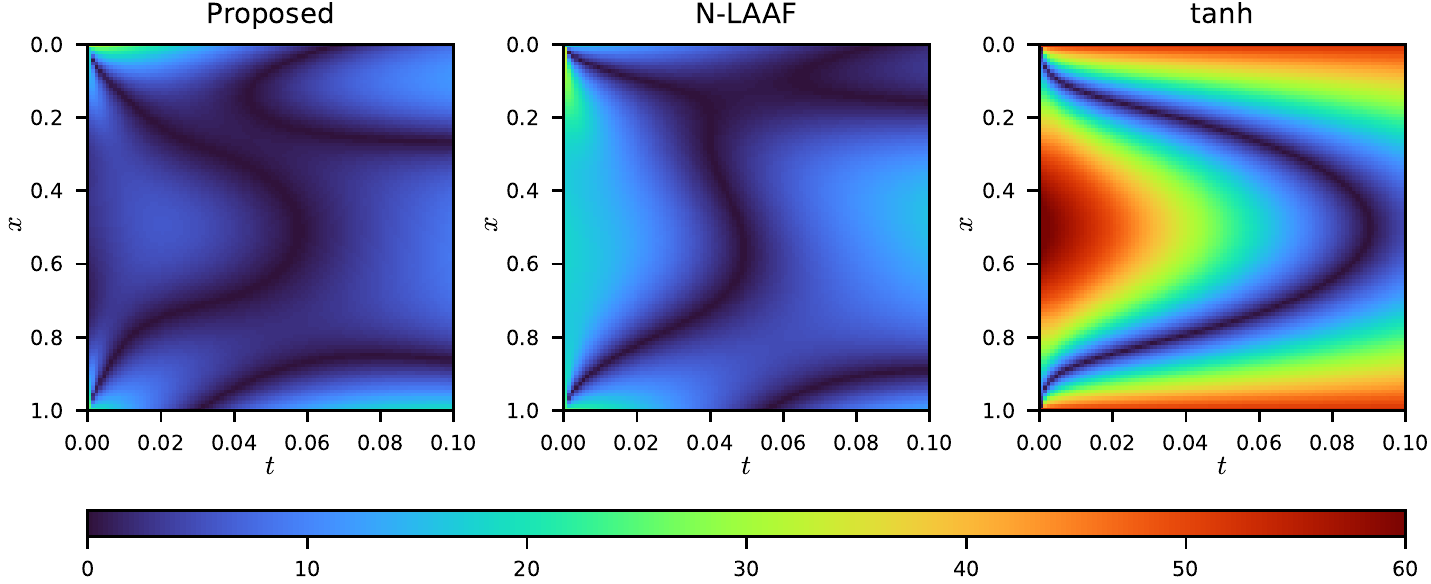}} \label{subfig: Transient 1D heat Transfer-errorMap} }  \caption{Results of  the Transient 1D heat Transfer equation: (a) Empirical convergence of training loss; (b) Empirical convergence of $\beta^i_k$; (c) Neural network solutions of the proposed method, N-LAAF, and tanh; (d) Neural network solution of N-LAAF; (e) Absolute error maps for neural network solutions of the proposed method, N-LAAF, and tanh.} \label{fig: performance of Transient 1D heat Transfer}
\end{figure}

Figure \ref{fig: performance of kappa in Transient 1D heat Transfer} plots the convergence of  absolute error $\left|\kappa-\kappa^*\right|$. Our method can converge to zero, which means that our approach can accurately identify the $\kappa^*$.  Interestingly, the benchmark methods output a negative value of $\kappa$, which is physically meaningless, and they diverge away from the true $\kappa^*$ value.  As a consequence of the absurd $\kappa$ values, the training losses of the benchmark methods hardly converge, whereas the proposed method reduces the loss function to lower orders of magnitude, as shown in Figure \ref{subfig: Transient 1D heat Transfer-training loss}.  Figure~\ref{subfig: Transient 1D heat Transfer-beta} shows the convergence behavior of $\beta^i_k$ in one of the neurons. Figures \ref{subfig: Transient 1D heat Transfer-solution} and \ref{subfig: Transient 1D heat Transfer-errorMap} correspond to the neural network solutions and absolute error maps for different methods. Our method can obtain the most accurate solution, while the solution from PINN with N-LAAF is close to ours. The reason for poor performance of the tanh function in this case can be attributed to both the vanishing gradient and the higher order of magnitude of output values.

\begin{table}[htbp!] 
\centering
\caption{Test performance of  forward and inverse problems for different methods in terms of MSE and RE.}\vspace{+0.4cm}
\begin{tabular}{lrrrrrr} 
\toprule
              & \multicolumn{2}{c}{tanh}                         & \multicolumn{2}{c}{N-LAAF}                     & \multicolumn{2}{c}{Proposed}                      \\ 
\cmidrule(lr){2-3}\cmidrule(lr){4-5}\cmidrule(lr){6-7}
              & \multicolumn{1}{c}{MSE} & \multicolumn{1}{c}{RE} & \multicolumn{1}{c}{MSE} & \multicolumn{1}{c}{RE} & \multicolumn{1}{c}{MSE} & \multicolumn{1}{c}{RE}  \\ 
\midrule
Eq.~\eqref{eq: 1D_1 differential equation}         & 295.10                  & 0.8757                 & 160.05                  & 0.6449                 & \textbf{0.00}                    & \textbf{0.0005}                  \\ 
\midrule
Eq.~\eqref{eq: 1D_2 differential equation}         & 0.40                    &        0.0088       & 0.21                    &        	0.0063       & \textbf{0.00}                    &      	\textbf{0.0000}             \\ 
\midrule
Eq.~\eqref{eq: 2D_1 differential equation}         & 2.61                    & 0.7725                 & 1.54                    & 0.5936                 & \textbf{0.55}                    & \textbf{0.3541}                  \\ 
\midrule
Eq.~\eqref{eq: heat transfer equation}      & 544.29                  & 0.3927                 & 564.08                  & 0.3998                 & \textbf{30.36}                   & \textbf{0.0928}                 \\
\bottomrule
\end{tabular} \label{tab: test loss summary 2}
\end{table}
The test performance  for different methods in terms of MSE and RE is summarized in Table~\ref{tab: test loss summary 2}.  The proposed method has the best performance in terms of MSE and RE. Based on the results in this Section, our proposed method can achieve the fastest training convergence and the best test performance for both forward and inverse problems using PINN.

\section{Conclusion}\label{sec: conclusion}
In this work, a Self-scalable tanh (Stan) is proposed for PINN, which contains a trainable parameter. The trainable parameter can be easily optimized together with NN weights and biases using gradient descent algorithms. It is shown theoretically that the proposed Stan with PINN has no spurious stationary points, which is also empirically verified in all case studies.  Our proposed Stan performs well on all types of case studies, including regression problems and a range of forward problems and an inverse problem. Especially, it is remarkably superior when the scales of the output are in higher orders of magnitude. Our proposed Stan can not only achieve faster convergence in  training but also show better generalization in prediction, compared to state-of-the-art activation functions. In the inverse problem, it is shown to be significantly useful for identifying the thermal diffusivity with the simulated sensor data.   

Though Stan improved the PINN to solve a wider range of problems, there are still some aspects of Stan that deserve further investigation. First, the  theoretical properties of the Stan function should be examined, particularly related to the gradient flow in PINN. Second, the initialization of the trainable parameters $\beta^i_k$'s needs additional research. The $\beta^i_k$'s were initialized with constant values for all the neurons in this work. The third aspect which needs development is the optimization process, particularly for the characteristic of the proposed Stan activation function.




	\appendix

\section{Proof of Theorem \ref{thm: no spurious}} \label{appendix: proofs}
\begin{proof} 
The above statements can be proved by contradiction. For simplicity, fix $w_{\mathcal{F}} = w_u = 1$. Suppose that the parameter vector $\tilde{\bmTheta}=\{\bm{W}^k,\bm{b}^k\}_{k=1}^D\bigcup\{\bm{\beta}^k\}_{k=1}^{D-1}$ is a limit point of $\{\tilde{\bmTheta}_m\}_{m\geq 0}$ and a  suboptimal stationary point. 

Let $\ell_f^j\coloneqq \phi^j(u_{\tilde{\bmTheta}}(\rho^j))$ and $\ell_u^j \coloneqq \left|u^j - u_{\tilde{\bmTheta}}(\bm{x}_u^j)\right|$. Denote $\bm{z}^{j,k}_f$ and $\bm{z}^{j,k}_u$ as the outputs of the $k$th layer for $\rho^j$ and $\bm{x}_u^j$, respectively. Now define
\begin{equation*}
    h_f^{i,j,k} \coloneqq \beta^i_k(\bm{W}^k_{i\cdot}\bm{z}^{j,k-1}_f+\bm{b}^{i,k}) \in \mathbb{R}
\end{equation*}
and 
\begin{equation*}
    h_u^{i,j,k} \coloneqq \beta^i_k(\bm{W}^k_{i\cdot}\bm{z}^{j,k-1}_u+\bm{b}^{i,k}) \in \mathbb{R},
\end{equation*}
for all $i=1,\dots,N_k$, where $\bm{W}^k_{i\cdot}\in\mathbb{R}^{1\times N_k}$ denotes $i$th row in $\bm{W}^k$ and $\bm{b}^{i,k}\in \mathbb{R}$.

Based on the proofs in~\citep[Propositions 1.2.1-1.2.4]{bertsekas1997nonlinear},  we have that  $\nabla J_{\gamma}(\tilde{\bmTheta})=\bm{0}$ and $J_{\gamma}(\tilde{\bmTheta})<J_{\gamma}(\bm{0})$. Since  $\nabla J_{\gamma}(\tilde{\bmTheta})=\bm{0}$, for all $k=1,\dots,D$ and all $i=1,\dots,N_k$, we have 
\begin{align}
        \frac{\partial J_{\gamma}(\tilde{\bmTheta})}{\partial \bm{W}^k_{i\cdot}} & =\frac{\beta_i^k}{N_f}\sum_{j=1}^{N_f}\frac{\partial \ell_f^j}{h_f^{i,j,k}}(\bm{z}^{j,k-1}_f)^\top + \frac{\beta_i^k}{N_u}\sum_{j=1}^{N_u}\frac{\partial\ell_u^j}{h_u^{i,j,k}}(\bm{z}^{j,k-1}_u)^\top =\bm{0} \label{eq: gradient W}\\
        \frac{\partial J_{\gamma}(\tilde{\bmTheta})}{\partial \bm{b}^{i,k}} & = \frac{\beta_i^k}{N_f}\sum_{j=1}^{N_f}\frac{\partial \ell_f^j}{h_f^{i,j,k}} + \frac{\beta_i^k}{N_u}\sum_{j=1}^{N_u}\frac{\partial\ell_u^j}{h_u^{i,j,k}} =0 \label{eq: gradient b}
\end{align}
In addition,  for all $k=1,\dots,D-1$ and all $i=1,\dots,N_k$, we have 
\begin{equation}\label{eq: gradient beta}
\begin{aligned}
         \frac{\partial J_{\gamma}(\tilde{\bmTheta})}{\partial \beta^i_k} &= \frac{1}{N_f}\sum_{j=1}^{N_f}\frac{\partial \ell_f^j}{h_f^{i,j,k}}(\bm{W}^k_{i\cdot}\bm{z}^{j,k-1}_f+\bm{b}^{i,k})+\frac{1}{N_u}\sum_{j=1}^{N_u}\frac{\partial \ell_u^j}{h_u^{i,j,k}}(\bm{W}^k_{i\cdot}\bm{z}^{j,k-1}_u+\bm{b}^{i,k}) + 2\gamma \beta^i_k \\
         & = 0.
\end{aligned}
\end{equation}
Multiply Eq.~\eqref{eq: gradient beta} by $\beta^i_k$, we have 
\begin{align}
   0= & \beta^i_k\frac{\partial J_{\gamma}(\tilde{\bmTheta})}{\partial \beta^i_k} \nonumber \\
   = &  \frac{\beta^i_k}{N_f}\sum_{j=1}^{N_f}\frac{\partial \ell_f^i}{h_f^{i,j,k}}(\bm{W}^k_{i\cdot}\bm{z}^{j,k-1}_f+\bm{b}^{i,k})+\frac{ \beta^i_k}{N_u}\sum_{j=1}^{N_u}\frac{\partial \ell_f^u}{h_u^{i,j,k}}(\bm{W}^k_{i\cdot}\bm{z}^{j,k-1}_u+\bm{b}^{i,k}) + 2\gamma (\beta^i_k)^2 \nonumber\\
   = &  \bm{W}^k_{i\cdot}(\frac{\partial J_{\gamma}(\tilde{\bmTheta})}{\partial \bm{W}^k_{i\cdot}})^{\top} + \bm{b}^{i,k} \frac{\partial J_{\gamma}(\tilde{\bmTheta})}{\partial \bm{b}^{i,k}} + 2\gamma (\beta^i_k)^2 \label{eq: massage}\\
   = & 2\gamma (\beta^i_k)^2, \nonumber
\end{align}
where Eq.~\eqref{eq: massage} comes from Eqs.~\eqref{eq: gradient W} and~\eqref{eq: gradient b}. It  implies that for all $\beta^i_k=0$ for all $i,k$ since $\gamma>0$. This shows $J_{\gamma}(\tilde{\bmTheta})=J_{\gamma}(\bm{0})$, which contradicts with $J_{\gamma}(\tilde{\bmTheta})<J_{\gamma}(\bm{0})$. 
\end{proof}

\bibliographystyle{apalike}
\spacingset{1}
\bibliography{IISETrans}

\end{document}